\documentclass[11pt]{article}

\usepackage{fullpage,times,url,bm}

\usepackage{amsthm,amsfonts,amsmath,amssymb,epsfig,color,float,graphicx,verbatim}
\usepackage{algorithm,algorithmic}
\usepackage{bbm}
\usepackage{natbib}

\usepackage{hyperref}
\hypersetup{
	colorlinks   = true, 
	urlcolor     = blue, 
	linkcolor    = blue, 
	citecolor   = black 
}

\newtheorem{theorem}{Theorem}[section]

\newtheorem{lemma}{Lemma}[section]
\newtheorem{corollary}{Corollary}[section]

\newtheorem{remark}{Remark}[section]

\usepackage{amssymb}

\usepackage{bbm}
\newcommand{\onefunc}{\mathbbm{1}}

\newcommand{\stam}[1]{}
\newcommand{\tup}[1]{\langle #1  \rangle}
\usepackage{mathtools}
\DeclarePairedDelimiter\ceil{\lceil}{\rceil}
\DeclarePairedDelimiter\floor{\lfloor}{\rfloor}

\newtheorem{assumption}[theorem]{Assumption}

\newcommand{\bx}{\mathbf{x}}
\newcommand{\bw}{\mathbf{w}}

\newcommand{\bb}{\mathbf{b}}

\newcommand{\bz}{\mathbf{z}}

\newcommand{\by}{\mathbf{y}}

\newcommand{\ca}{{\cal A}}
\newcommand{\cb}{{\cal B}}
\newcommand{\cd}{{\cal D}}

\newcommand{\cg}{{\cal G}}
\newcommand{\ch}{{\cal H}}

\newcommand{\cl}{{\cal L}}
\newcommand{\cf}{{\cal F}}

\newcommand{\cx}{{\cal X}}
\newcommand{\cy}{{\cal Y}}
\newcommand{\cz}{{\cal Z}}
\newcommand{\cs}{{\cal S}}
\newcommand{\cn}{{\cal N}}

\DeclareMathOperator*{\sign}{sign}

\newcommand{\reals}{{\mathbb R}}

\DeclareMathOperator{\poly}{poly}
\DeclareMathOperator{\polylog}{polylog}
\DeclareMathOperator*{\E}{\mathbb{E}}

\DeclareMathOperator{\xormaj}{XOR-MAJ}
\DeclareMathOperator{\maj}{MAJ}
\DeclareMathOperator{\xor}{XOR}

\newcommand{\tn}{{\tilde{n}}}
\newcommand{\inner}[1]{\langle #1 \rangle}

\title{From Local Pseudorandom Generators to Hardness of Learning}

\author{
    Amit Daniely\thanks{School of Computer Science and Engineering, The Hebrew University, Jerusalem, Israel and Google Research Tel-Aviv, \texttt{amit.daniely@mail.huji.ac.il }}
	\and
	Gal Vardi\thanks{Weizmann Institute of Science, Israel, \texttt{gal.vardi@weizmann.ac.il}}
}
\date{}

\begin{document}

\maketitle

\begin{abstract}
We prove hardness-of-learning results under a well-studied assumption on the existence of local pseudorandom generators. As we show, this assumption allows us to surpass the current state of the art, and prove hardness of various basic problems, with no hardness results to date.

Our results include: hardness of learning shallow ReLU neural networks under the Gaussian distribution and other distributions; hardness of learning intersections of $\omega(1)$ halfspaces, DNF formulas with $\omega(1)$ terms, and ReLU networks with $\omega(1)$ hidden neurons; hardness of weakly learning deterministic finite automata under the uniform distribution; hardness of weakly learning depth-$3$ Boolean circuits under the uniform distribution, as well as distribution-specific hardness results for learning DNF formulas and intersections of halfspaces. We also establish lower bounds on the complexity of learning intersections of a constant number of halfspaces, and ReLU networks with a constant number of hidden neurons. Moreover, our results imply the hardness of virtually all improper PAC-learning problems (both distribution-free and distribution-specific) that were previously shown hard under other assumptions.
\end{abstract}

\section{Introduction}

The computational complexity of PAC learning has been extensively studied over the past decades. Nevertheless, for many learning problems there is still a large gap between the complexity of the best known algorithms and the hardness results. The situation is even worse for {\em distribution-specific learning}, namely, where the inputs are drawn from some known distribution (e.g., the uniform or the normal distribution). Since there are very few distribution-specific hardness results, the status of most basic learning problems with respect to natural distributions is wide open. 

The main obstacle for achieving hardness results for learning problems, is the ability of a learning algorithm to return
a hypothesis which does not belong to the considered hypothesis class (such an algorithm is called {\em improper learner}).
This flexibility makes it very difficult to apply reductions from $\textsf{NP}$-hard problems, and unless we face a dramatic breakthrough in complexity theory, it seems unlikely that hardness of improper learning can be established on standard complexity assumptions (see \cite{ApplebaumBaXi08,daniely2013average}). Indeed, all currently known lower bounds are based on assumptions from cryptography or average-case hardness.

In this work, we consider hardness of learning under assumptions 
on the existence of {\em local pseudorandom generators (PRG)} with polynomial stretch. This type of assumptions was extensively studied in the last two decades. Under such assumptions we extend the current state of the art, and establish new hardness results for several hypothesis classes, for both distribution-free and distribution-specific learning. Our results apply to fundamental classes, such as DNFs, Boolean circuits, intersections of halfspaces, neural networks and automata. 
Most of the results are based on the mere assumption that {\em some} local PRG with polynomial stretch exists. The only exceptions are our lower bounds for intersections of a constant number of halfspaces, and neural networks with a constant number of neurons, that are based on a stronger assumption, regarding a specific candidate for such a PRG, that was suggested by \cite{applebaum2016algebraic}. Below we discuss our results and related work.

\paragraph{DNFs and Boolean circuits.}
Learning polynomial-size DNF formulas has been a major effort in computational learning theory.
The best known upper bound for (distribution-free) learning of polynomial-size DNF formulas over $n$ variables is $2^{\tilde{O}(n^{1/3})}$, due to \cite{klivans2001learning}.
Already in Valiant's seminal paper \citep{Valiant84}, it is shown that for every constant $q$, DNF formulas with $q$ terms can be learned efficiently.
Hardness of improperly learning DNF formulas is implied by \cite{applebaum2010public} under a combination of two assumptions: the first is related to the planted dense subgraph problem in hypergraphs, and the second is related to local PRGs.
\cite{daniely2016complexity} showed hardness of improperly learning DNF formulas with $q(n)=\omega(\log(n))$ terms, under 
a common assumption, namely,
that refuting a random $K$-SAT formula is hard.
We improve this lower bound, and show hardness of learning DNF formulas with $q(n)=\omega(1)$ terms.

\cite{LinialMaNi93} gave a quasi-polynomial ($O(n^{\polylog(n)})$) upper bound for learning constant-depth Boolean circuits ($\textsf{AC}^0$) on the uniform distribution. Their result was later improved
to a slightly better quasi-polynomial bound \citep{boppana1997average,haastad2001slight}.
Learning $\textsf{AC}^0$ in quasi-polynomial time under other restricted distributions was studied in, e.g., 
\cite{furst1991improved,blais2010polynomial}.
In  \cite{Kharitonov93} it is shown, under a relatively strong assumption on the complexity of factoring random Blum integers, that learning depth-$d$ circuits on the uniform distribution is hard, where $d$ is an unspecified sufficiently large constant. 
 \cite{applebaum2016fast} showed, under an assumption on a specific candidate for Goldreich's PRG (based on the $\xormaj$ predicate), that learning depth-$3$ Boolean circuits under the uniform distribution is hard.
We prove distribution-specific hardness of improperly learning Boolean circuits of depth-$2$ (namely, DNFs) and depth-$3$. For DNF formulas with $n^\epsilon$ terms, we show hardness of learning on a distribution where each component is drawn i.i.d. from a Bernoulli distribution (which is not uniform). For depth-$3$ Boolean circuits, we show hardness of weak learning on the uniform distribution (recall that we only assume here the existence of \emph{some} local PRG, rather than a specific candidate).

\paragraph{Intersections of halfspaces.}
Learning intersections of halfspaces is also a fundamental problem in learning theory.
\cite{KlivansSh06} showed, assuming the hardness of the shortest vector problem, that improper learning of intersections of $n^\epsilon$ halfspaces for a constant $\epsilon>0$, is hard. The hardness result from \cite{daniely2016complexity} for learning DNF formulas with $\omega(\log(n))$ terms, implies hardness of learning intersections of $\omega(\log(n))$ halfspaces, since every DNF formula with $q(n)$ terms can be realized by the complement of an intersection of $q(n)$ halfspaces.
Our result on hardness of learning DNF formulas with $\omega(1)$ terms implies hardness of learning intersections of $\omega(1)$ halfspaces, and thus improves the bound from \cite{daniely2016complexity}.
Learning intersections of halfspaces under some restricted distributions has been studied in, e.g., \cite{baum1990polynomial,blum1997learning,vempala1997random,klivans2004learning,klivans2009baum}. Our distribution-specific hardness result for DNFs implies a first distribution-specific hardness result for improperly learning intersections of $n^\epsilon$ halfspaces.

Efficient algorithms for (distribution-free) learning intersections of $k$ halfspaces are not known even for a constant $k$, and even for $k=2$. \cite{klivans2004learning} showed an algorithm for distribution-free learning $k$ weight-$w$ halfspaces on the hypercube in time $n^{O(k\log(k)\log(w))}$, where the weight of a halfspace is the sum of the absolute values of its components.
We study distribution-free improper learning of a constant number of halfspaces, namely, where the number $k$ of halfspaces is independent of $n$. We show (under our stronger assumption regarding a specific candidate for a local PRG) a 
$n^{\beta k}$
lower bound.
More formally, we show that there is an absolute constant $\beta>0$ (independent of $k,n$), such that learning intersections of $k$ halfspaces on the hypercube within a constant error requires time $\Omega(n^{\beta k})$.
Also, a conjecture due to \cite{applebaum2016algebraic} implies that our lower bound holds for, e.g., $\beta = \frac{1}{11}$.
This is the first lower bound for improperly learning intersections of a constant number of halfspaces.

\paragraph{Neural networks.}
Hardness of improperly learning neural networks (with respect to the square loss) follows from hardness of learning intersection of halfspaces. Hence, the results from \cite{KlivansSh06} and \cite{daniely2016complexity} imply hardness of improperly learning depth-$2$ neural networks with $n^\epsilon$ and $\omega(\log(n))$ hidden neurons (respectively). \cite{daniely2020hardness} showed, under the assumption that refuting a random $K$-SAT formula is hard, that improperly learning depth-$2$ neural networks is hard already if its weights are drawn from some ``natural" distribution or satisfy some ``natural" properties. While hardness of proper learning is implied by hardness of improper learning, there are some recent works that show hardness of properly learning depth-$2$ networks under more standard assumptions (cf. \cite{goel2020tight}).  

Our hardness results for DNFs and intersections of halfspaces imply new hardness results for learning neural networks.
We show hardness of improperly learning depth-$2$ neural networks with $\omega(1)$ hidden neurons and the ReLU activation function, with respect to the square loss. Thus, we improve the $\omega(\log(n))$ lower bound implied by \cite{daniely2016complexity}. Moreover, the lower bound implied by \cite{daniely2016complexity} requires an activation function also in the output neuron, while our lower bound does not.
For depth-$2$ networks with a constant number $k$ of hidden neuron, namely, where the number of hidden neurons is independent of $n$, we show (under our stronger assumption regarding a specific candidate for a local PRG) a 
$\Omega(n^{\beta k})$
lower bound, 
where $\beta$ is a constant independent of $n,k$.
This is the first lower bound for improperly learning neural networks with a constant number of hidden neurons.

Due to the empirical success of neural networks, there has been much effort to
understand under what assumptions neural networks may be learned efficiently. This effort includes making assumptions on the input distribution~\citep{li2017convergence,brutzkus2017globally,du2017convolutional,du2017gradient,du2018improved,goel2018learning}, the network's weights~\citep{arora2014provable,das2019learnability,agarwal2020deep,goel2017learning}, or both~\citep{janzamin2015beating,tian2017analytical,bakshi2019learning}. Hence, distribution-specific learning of neural networks is a central problem. Several works in recent years have shown hardness of distribution-specific learning shallow neural networks using gradient-descent or statistical query (SQ) algorithms \citep{shamir2018distribution,song2017complexity,vempala2019gradient,goel2020superpolynomial,diakonikolas2020algorithms}. We note that while the SQ framework captures the gradient-descent algorithm, it does not capture, for example, stochastic gradient-descent (SGD), which examines training points individually (see a discussion in \cite{goel2020superpolynomial}).
Distribution-specific hardness of learning a single ReLU neuron in the agnostic setting was studied in \cite{goel2019time,goel2020statistical,diakonikolas2020near}.

We show hardness of improper distribution-specific learning of depth-$2$ and depth-$3$ ReLU neural networks with respect to the square loss.
First, our distribution-specific hardness results for Boolean circuits, imply hardness of learning depth-$2$ networks on a distribution where each component is drawn i.i.d. from a (non-uniform) Bernoulli distribution, and depth-$3$ networks on the uniform distribution on the hypercube. More importantly, we also show hardness of improperly learning depth-$3$ networks on the standard Gaussian distribution. 

\paragraph{Automata.}
Deterministic finite automata are an elementary computational model, and their learnability is a classical problem in learning theory.
An efficient algorithm due to \cite{angluin1987learning} is known for learning deterministic automata with membership and equivalence queries, and was extensively studied over the last decades.
Improper learning of deterministic automata with $n^\epsilon$ states is known to be harder than breaking the RSA cryptosystem, factoring Blum integers and detecting quadratic residues \citep{KearnsVa94}. It is also harder than refuting a random $K$-SAT formula \citep{daniely2016complexity}.
The question of whether deterministic automata are learnable on the uniform distribution was posed by \cite{pitt1989inductive} over $30$ years ago, and remained open (cf. \cite{fish2017open,michaliszyn2019approximate}). 
We solve this problem, by showing
hardness of weakly learning deterministic automata on the uniform distribution over the hypercube.
This is the first distribution-specific hardness result for improperly learning automata.

\paragraph{Other classes.}
Our lower bound for learning DNF formulas with $\omega(1)$ terms implies hardness of learning $\omega(1)$-sparse polynomial threshold functions over $\{0,1\}^n$, where a $q$-sparse polynomial has at most $q$ monomials with non-zero coefficients. It improves the lower bound from \cite{daniely2016complexity} for learning $\omega(\log(n))$-sparse polynomial threshold functions.
Also, we show hardness of learning $\omega(1)$-sparse $GF(2)$ polynomials over $\{0,1\}^n$.
Subexponential-time upper bounds for these problems are given in \cite{hellerstein2007pac}.

Finally, our lower bound for learning DNFs implies hardness of agnostically learning conjunctions, halfspaces and parities. These problems are already known to be hard under other assumptions \citep{feldman2006new,daniely2016half,blum2003noise,daniely2016complexity}.

\paragraph{A summary of our contribution.}
Below we summarize our main contributions:
\begin{itemize}
\item Hardness of learning DNF formulas with $\omega(1)$ terms.
\item Distribution-specific hardness of learning DNFs and of weakly learning depth-$3$ Boolean circuits.
\item Hardness of learning intersections of $\omega(1)$ halfspaces.
\item Distribution-specific hardness of learning intersections of halfspaces on the hypercube.
\item $\Omega(n^{\beta k})$-time lower bound for learning intersections of a constant number $k$ of halfspaces, where $\beta$ is an absolute constant.
\item Hardness of learning depth-$2$ neural networks with $\omega(1)$ hidden neuron.
\item $\Omega(n^{\beta k})$-time lower bound for learning depth-$2$ neural networks with a constant number $k$ of hidden neurons, where $\beta$ is an absolute constant.
\item Distribution-specific hardness of learning depth-$2$ and~$3$ neural networks on the hypercube.
\item Distribution-specific hardness of learning depth-$3$ neural networks on the standard Gaussian distribution.
\item Distribution-specific hardness of weakly learning deterministic automata on the hypercube.
\item Hardness of learning $\omega(1)$-sparse polynomial threshold functions and $\omega(1)$-sparse $GF(2)$ polynomials over $\{0,1\}^n$.
\item Hardness of agnostically learning conjunctions, halfspaces and parities (these problems are already known to be hard under other assumptions).
\item Our results imply the hardness of virtually all\footnote{It does not imply the hardness result from \cite{daniely2020hardness} for learning depth-$2$ neural networks whose weights are drawn from some ``natural" distribution.} improper PAC-learning problems (both distribution-free and distribution-specific) that were previously shown hard (under various complexity assumptions). Moreover, our technique is simple, and we believe that it might be useful for showing hardness of more learning problems in the future.
\end{itemize}

Our paper is structured as follows: In Section~\ref{sec:preliminaries} we provide necessary notations and definitions, and discuss our assumptions. The results are stated in Section~\ref{sec:results}. We informally sketch our proof technique in Section~\ref{sec:technique}, with all formal proofs deferred to the appendix.

\section{Preliminaries}
\label{sec:preliminaries}

\subsection{Notations}

We use bold-faced letters to denote vectors, e.g., $\bx=(x_1,\ldots,x_d)$.
For a vector $\bx$ and a sequence $S=(i_1,\ldots,i_k)$ of $k$ indices, we let $\bx_S=(x_{i_1},\ldots,x_{i_k})$, i.e., the restriction of $\bx$ to the indices $S$.
We denote by $\onefunc(\cdot)$ the indicator function, for example $\onefunc(t \geq 5)$ equals $1$ if $t \geq 5$ and $0$ otherwise.
For an integer $d \geq 1$ we denote $[d]=\{1,\ldots,d\}$.
The majority predicate $\maj_k:\{0,1\}^k \rightarrow \{0,1\}$ is defined by $\maj_k(\bx)=1$ iff $\sum_{i \in [k]}x_i > \frac{k}{2}$.
We denote $\xor_k:\{0,1\}^k \rightarrow \{0,1\}$ where $\xor_k(\bx)=x_1 \oplus \ldots \oplus x_k$.
For $m \in \reals$ we let $\sign(m)=1$ if $m>0$ and $\sign(m)=0$ otherwise.

\subsection{Local pseudorandom generators}

An $(n,m,k)$-hypergraph is a hypergraph over $n$ vertices $[n]$ with $m$ hyperedges $S_1,\ldots,S_m$, each of cardinality $k$. Each hyperedge $S = (i_1,\ldots,i_k)$ is ordered, and all the $k$ members of a hyperedge are distinct.
We let $\cg_{n,m,k}$ be the distribution over such hypergraphs in which a hypergraph is chosen by picking each hyperedge uniformly and independently at random among all the possible $n \cdot (n-1) \cdot \ldots \cdot (n-k+1)$ ordered hyperedges.
Let $P:\{0,1\}^k \rightarrow \{0,1\}$ be a predicate, and let $G$ be a $(n,m,k)$-hypergraph.
We call {\em Goldreich's pseudorandom generator (PRG)} \citep{goldreich2000candidate} the function $f_{P,G}:\{0,1\}^n \rightarrow \{0,1\}^m$ such that for $\bx \in \{0,1\}^n$, we have $f_{P,G}(\bx) = (P(\bx_{S_1}),\ldots,P(\bx_{S_m}))$.
The integer $k$ is called the {\em locality} of the PRG. If $k$ is a constant then the PRG and the predicate $P$ are called {\em local}.
We say that the PRG has {\em polynomial stretch} if $m = n^s$ for some constant $s>1$.
We let $\cf_{P,n,m}$ denote the collection of functions $f_{P,G}$ where $G$ is an $(n,m,k)$-hypergraph. We sample a function from $\cf_{P,n,m}$ by choosing a random hypergraph $G$ from $\cg_{n,m,k}$.

We denote by $G \xleftarrow{R} \cg_{n,m,k}$ the operation of sampling a hypergraph $G$ from $\cg_{n,m,k}$, and by $\bx \xleftarrow{R} \{0,1\}^n$ the operation of sampling $\bx$ from the uniform distribution on $\{0,1\}^n$.
We say that $\cf_{P,n,m}$ is $\varepsilon$-pseudorandom generator ($\varepsilon$-PRG) if for every polynomial-time probabilistic algorithm $\ca$ the {\em distinguishing advantage}
\[
\left|
\Pr_{G \xleftarrow{R} \cg_{n,m,k}, \bx \xleftarrow{R} \{0,1\}^n}[\ca(G,f_{P,G}(\bx))=1] -
\Pr_{G \xleftarrow{R} \cg_{n,m,k}, \by \xleftarrow{R} \{0,1\}^m}[\ca(G,\by)=1]
\right|
\]
is at most $\varepsilon$. Thus, the distinguisher $\ca$ is given a random hypergraph $G$ and a string $\by \in \{0,1\}^m$, and its goal is to distinguish between the case where $\by$ is chosen at random, and the case where $\by$ is a random image of $f_{P,G}$.

Our main assumption is that local PRGs with polynomial stretch and constant distinguishing advantage exist:
\begin{assumption}
\label{ass:localPRG}
For every constant $s>1$, there exists a constant $k$ and a predicate $P:\{0,1\}^k \rightarrow \{0,1\}$, such that $\cf_{P,n,n^s}$ is $\frac{1}{3}$-PRG.
\end{assumption}

Note that we assume constant distinguishing advantage. In the literature, a requirement of negligible distinguishing advantage\footnote{More formally, that for $1-o_n(1)$ fraction of the hypergraphs, the distinguisher has no more than negligible advantage.} is often considered (cf. \cite{applebaum2016algebraic,applebaum2016cryptographic,couteau2018concrete}). Thus, our requirement from the PRG is weaker.

Local PRGs have been extensively studied in the last two decades.
In particular, local PRGs with polynomial stretch have shown to have remarkable applications, such as secure-computation with constant computational overhead \citep{ishai2008cryptography,applebaum2017secure}, and general-purpose obfuscation based on constant degree multilinear maps (cf. \cite{lin2016indistinguishability,linVai2016indistinguishability}).
A significant evidence for Assumption~\ref{ass:localPRG} was shown in \cite{applebaum2013pseudorandom}. He showed that Assumption~\ref{ass:localPRG} follows from the assumption that for every constant $s>1$, there exists a {\em sensitive} local predicate\footnote{A predicate is sensitive if at least one coordinate $i$ has full influence, i.e., flipping the value of the $i$-th variable always changes the output.} $P$ such that $\cf_{P,n,n^s}$ is one-way. This is a variant of Goldreich's one-wayness assumption \citep{goldreich2000candidate}.

In light of Assumption~\ref{ass:localPRG}, an important question is which local predicates are secure.
\cite{odonnell2014goldreich} showed that a property called {\em resiliency} yields pseudorandomness against attacks which are based on a large class of semidefinite programs.
\cite{FeldmanPeVe2015} showed that resiliency also ensures pseudorandomness against a wide family of statistical algorithms.
\cite{applebaum2016algebraic} showed that predicates with high resiliency and high {\em rational degree} are secure against two classes of distinguishing attacks: {\em linear attacks} and {\em algebraic attacks}. These classes include all known attacks against PRGs.
Furthermore, 
they suggested the following predicate as a candidate for local PRG with polynomial stretch:
\[
\xormaj_{a,b}(\bz) = (z_1 \oplus \ldots \oplus z_a) \oplus \maj_b(z_{a+1},\ldots,z_{a+b})~.
\]
By their conjecture, for every constant $s>1$ and constants $a \geq 5s$ ,$b > 36s$ the predicate $P=\xormaj_{a,b}$ is such that the collection $\cf_{P,n,n^s}$ is PRG with negligible distinguishing advantage.
This predicate has high resiliency and rational degree, and is secured against all known attacks.
Its security has been studied also in \cite{couteau2018concrete,meaux2019improved,applebaum2016fast}.
We make a somewhat weaker assumption:
\begin{assumption}
\label{ass:xormaj}
There is a constant $\alpha>0$, such that for every constant $s>1$ there is a constant $l$ such that for the predicate $P=\xormaj_{\ceil{\alpha s},l}$ the collection $\cf_{P,n,n^s}$ is $\frac{1}{3}$-PRG.
\end{assumption}

Our results on learning intersections of a constant number of halfspaces, and on leaning neural networks with a constant number of hidden neurons, rely on Assumption~\ref{ass:xormaj}. All other results rely on Assumption~\ref{ass:localPRG}. Thus, most of our results assume the existence of a local PRG with polynomial stretch, and do not rely on a specific candidate.

In our assumptions we consider local PRGs that are secure against polynomial-time algorithms. Hence, the hardness results in this paper rule out polynomial-time learning algorithms. We note that our results can be improved by strengthening the assumptions, e.g., by assuming that the local PRGs are secure against some quasi-polynomial time algorithms. 

\textbf{Prior works on the relation between Goldreich's PRG and hardness of learning.} 
First, Goldreich's PRG are closely related to CSP refutation, and there have been many works on the relation between CSP refutation and hardness of learning (e.g., \cite{daniely2013average, daniely2016complexity, daniely2016half, vadhan2017learning, kothari2018improper,daniely2020hardness}). Moreover, some applications of variants of Goldreich's assumption and local PRGs for hardness of learning are shown in \cite{applebaum2010public,applebaum2016fast,nanashima2020extending}.

\subsection{PAC learning}

A {\em hypothesis class} $\ch$ is a series of collections of functions $\ch_n \subset \cy^{\cx_n},\;n=1,2,\ldots$. We often abuse notation and identify $\ch$ with $\ch_n$.
The domain sets $\cx_n$ we consider are $\{0,1\}^n$ or $\reals^n$, and the label sets $\cy$ we consider are $\{0,1\}$ or $\reals$. Let $\cz_n = \cx_n \times \cy$ and let $\cd_n$ be a distribution on $\cz_n$.
A {\em loss function} is a mapping $\ell:\ch_n \times \cz_n \rightarrow \reals_+$. We consider the following loss functions. The {\em 0-1 loss} is $\ell_{0-1}(h,(\bx,y))=\onefunc(h(\bx) \neq y)$.
For $\cy=\reals$, the {\em square loss} is $\ell_{\text{sq}}(h,(\bx,y))=(h(\bx)-y)^2$.
The {\em error} of $h:\cx_n \rightarrow \cy$ is $L_{\cd_n}(h) = \E_{\bz \in \cd_n}[\ell(h,\bz)]$. Note that for the 0-1 loss we have $L_{\cd_n}(h) = \Pr_{(\bx,y) \sim \cd_n}\left[h(\bx) \neq y\right]$.
For a class $\ch_n$, we let $L_{\cd_n}(\ch_n)=\min_{h \in \ch_n}L_{\cd_n}(h)$.
We say that $\cd_n$ is {\em realizable} by $h$ (respectively $\ch_n$) if $L_{\cd_n}(h)=0$ (respectively $L_{\cd_n}(\ch_n)=0$).


A {\em learning algorithm} $\cl$ is given $\epsilon,\delta \in (0,1)$, as well as an oracle access to examples from an unknown distribution $\cd$ on $\cz_n$. It should output a (description of) hypothesis $h:\cx_n\ \rightarrow \cy$.
We say that $\cl$ {\em (PAC) learns} $\ch$, if for every realizable $\cd$, with probability at least $1-\delta$, the algorithm $\cl$ outputs a hypothesis with error at most $\epsilon$.
We say that $\cl$ {\em agnostically learns} $\ch$, if for every $\cd$, with probability at least $1-\delta$, the algorithm $\cl$ outputs a hypothesis with error at most $L_{\cd}(\ch)+\epsilon$.
Note that by these definitions, $\cl$ should succeed for every realizable distribution $\cd$ (in the former definition) or for every distribution $\cd$ (in the later definition). Hence, this setting is called {\em distribution-free learning}.
We now consider {\em distribution-specific learning}, namely, where the marginal distribution of $\cd$ on $\cx_n$ is fixed.
Let $\cd_\cx$ be a distribution on $\cx_n$.
We say that $\cl$ {\em learns $\ch$ on $\cd_\cx$}, if for every realizable $\cd$ whose marginal distribution on $\cx_n$ is $\cd_\cx$, with probability at least $1-\delta$, the algorithm $\cl$ outputs a hypothesis with error at most $\epsilon$.

When the error is defined with respect to the 0-1 loss, we also consider a weaker requirement from $\cl$.
For $\gamma>0$ we say that $\cl$ {\em $\gamma$-weakly learns} $\ch$, if for every realizable $\cd$, the algorithm $\cl$ is given $\delta \in (0,1)$, and outputs with probability at least $1-\delta$ a hypothesis with error at most $\frac{1}{2}-\gamma$.
We say that $\cl$ {\em $\gamma$-weakly learns $\ch$ on $\cd_\cx$}, if for every realizable $\cd$ whose marginal distribution on $\cx_n$ is $\cd_\cx$, the algorithm $\cl$ is given $\delta \in (0,1)$, and outputs with probability at least $1-\delta$ a hypothesis with error at most $\frac{1}{2}-\gamma$.
Thus, when $\gamma$ is small, the returned hypothesis needs to be at least slightly better than a random guess.

We say that $\cl$ is {\em efficient} if it runs in time $\poly(n,1/\epsilon,1/\delta)$ (or $\poly(n,1/\delta)$, for weak learning), and outputs a hypothesis that can be evaluated in time $\poly(n,1/\epsilon,1/\delta)$ (respectively, $\poly(n,1/\delta)$).
Finally, $\cl$ is {\em proper} if it always outputs a hypothesis in $\ch$. Otherwise, we say that $\cl$ is {\em improper}.

By boosting results \citep{Schapire89,Freund95}, if there is an efficient algorithm that $\frac{1}{n^c}$-weakly learns $\ch$ for some constant $c>0$, then there is also an efficient improper algorithm that learns $\ch$. Hence, in the distribution-free setting, hardness of improper learning implies hardness of improper weak learning. These boosting arguments do not apply to the distribution-specific setting.

\subsection{Neural networks}

We consider feedforward neural networks, computing functions from $\reals^n$ to $\reals$. The network is composed of layers of neurons, where each neuron computes a function of the form $\bx \mapsto \sigma(\bw^{\top}\bx+b)$, where $\bw$ is a weight vector, $b$ is a bias term and $\sigma: \reals \mapsto \reals$ is a non-linear activation function. In this work we focus on the ReLU activation function, namely, $\sigma(z) = [z]_+ = \max\{0,z\}$. For a matrix $W = (\bw_1,\ldots,\bw_d)$, we let $\sigma(W^\top \bx+\bb)$ be a shorthand for $\left(\sigma(\bw_1^{\top}\bx+b_1),\ldots,\sigma(\bw_d^{\top}\bx+b_d)\right)$, and define a layer of $d$ neurons as $\bx \mapsto \sigma(W^\top \bx+\bb)$. By denoting the output of the $i$-th layer as $O_i$, we can define a network of arbitrary depth recursively by $O_{i+1}=\sigma(W_{i+1}^\top O_i+\bb_{i+1})$.
The {\em weights vector} of the $j$-th neuron in the $i$-th layer is the $j$-th column of $W_i$, and its {\em outgoing-weights vector} is the $j$-th row of $W_{i+1}$.
The {\em fan-in}
of a neuron is the number of non-zero entries in its weights vector.
We define the \emph{depth} of the network as the number of layers.
Unless stated otherwise, the output neuron also has a ReLU activation function.
A neuron which is not an input or output neuron is called a {\em hidden neuron}.
We sometimes consider neural networks with multiple outputs.

\subsection{Automata}

A {\em deterministic finite automaton\/} (DFA, for short) is a tuple $A=\tup{\Sigma,Q,q_0,\delta,F}$, where $\Sigma$ is a finite alphabet, $Q$ is a finite set of states, $q_0 \in Q$ is the initial state, $\delta:  Q \times \Sigma \rightarrow Q$ is a transition function, and $F \subseteq Q$ is a set of final states.
Given a word $w=\sigma_1 \cdot \sigma_2 \cdots \sigma_l \in \Sigma^*$, the {\em run\/} of $A$ on $w$ is the sequence $r=q_0,q_1,\ldots,q_l$ of states such that $q_{i+1}=\delta(q_i,\sigma_{i+1})$ for all $i \geq 0$. The run is accepting if $q_l \in F$.
The DFA $A$ {\em accepts} the word $w$ iff the run of $A$ on $w$ is accepting. We sometimes use the notation $A(w)=1$ (respectively, $A(w)=0$) to indicate that $A$ accepts (respectively, rejects) $w$.
The {\em size} of $A$ is the number of its states.

\section{results}
\label{sec:results}

\subsection{DNFs and Boolean circuits}

In the following theorem we show distribution-free hardness for DNF formulas with $\omega(1)$ terms, and distribution-specific hardness for DNF formulas with $n^\epsilon$ terms (see proof in Appendix~\ref{app:proof DNF}).

\begin{theorem}
\label{thm:DNF}
Under Assumption~\ref{ass:localPRG}, for every $q(n)=\omega(1)$, there is no efficient algorithm that learns DNF formulas with $n$ variables and $q(n)$ terms. Moreover, for every constant $\epsilon>0$, there is no efficient algorithm that learns DNF formulas with $n^\epsilon$ terms, on a distribution where each component is drawn i.i.d. from a (non-uniform) Bernoulli distribution.
\end{theorem}

Theorem~\ref{thm:DNF} gives distribution-specific hardness for learning DNF formulas, namely, depth-$2$ Boolean circuits, where the input distribution is such that the components are i.i.d. copies from a Bernoulli distribution. For depth-$3$ Boolean circuits we show hardness of weak learning, where the input distribution is uniform on the hypercube (see proof in Appendix~\ref{app:proof circuit}).

\begin{theorem}
\label{thm:circuit}
Under Assumption~\ref{ass:localPRG}, for every constants $\gamma,\epsilon>0$, there is no efficient algorithm that $\gamma$-weakly learns depth-$3$ Boolean circuits of size $n^\epsilon$ on the uniform distribution over $\{0,1\}^n$.
\end{theorem}

\subsection{Intersections of halfspaces}

Any function realized by a DNF formula with $q(n)$ terms can be also realized by the complement of an intersection of $q(n)$ halfspaces. Hence, Theorem~\ref{thm:DNF} implies the following corollary.

\begin{corollary}
\label{cor:intersections}
Under Assumption~\ref{ass:localPRG}, for every $q(n)=\omega(1)$, there is no efficient algorithm that learns intersections of $q(n)$ halfspaces over $\{0,1\}^n$. Moreover, for every constant $\epsilon>0$, there is no efficient algorithm that learns intersections of $n^\epsilon$ halfspaces, on a distribution where each component is drawn i.i.d. from a (non-uniform) Bernoulli distribution.
\end{corollary}

We now consider intersections of a constant number $k$ of halfspaces (i.e., $k$ is independent of $n$), and show a 
$\Omega(n^{\beta k})$
lower bound (see proof in Appendix~\ref{app:proof intersections}).

\begin{theorem}
\label{thm:intersections}
Let $\ch \subseteq \{0,1\}^{(\{0,1\}^n)}$ be the functions expressible by intersections of $k$ halfspaces, where $k$ is a constant independent of $n$.
Let $\cl$ be a learning algorithm, that for every $\ch$-realizable distribution, returns with probability at least $\frac{3}{4}$ a hypothesis with error at most $\frac{1}{10}$.
Then, under Assumption~\ref{ass:xormaj}, there is a universal constant $\beta>0$ (independent of $k,n$) such that the time-complexity of $\cl$ is $\Omega(n^{\beta k})$.
\end{theorem}


\begin{remark}
\cite{applebaum2016algebraic} conjectured that Assumption~\ref{ass:xormaj} holds for $\alpha \geq 5$. It implies that Theorem~\ref{thm:intersections} holds for, e.g., $\beta = \frac{1}{11}$.
\end{remark}

\subsection{Neural networks}

We consider neural networks with the ReLU activation function.
Since neural networks are real-valued, we consider here the square loss rather than the 0-1 loss.
Our results hold for networks where the norms of the weights of each neuron are bounded by some $\poly(n)$. By a simple scaling trick (i.e., by increasing the input dimension), it follows that for every constant $\epsilon>0$, the results also hold for networks where the norms of the weights of every neuron are bounded by $n^\epsilon$.

From Theorems~\ref{thm:DNF} and~\ref{thm:circuit}, it is not hard to show the following theorems (see proofs in Appendix~\ref{app:proof nn discrete} and~\ref{app:proof nn fixed k}).

\begin{theorem}
\label{thm:nn discrete}
Under Assumption~\ref{ass:localPRG}, we have:
\begin{enumerate}
\item For every $q(n)=\omega(1)$, there is no efficient algorithm that learns depth-$2$ neural networks with $q(n)$ hidden neurons, and no activation function in the output neuron, where the input distribution is supported on $\{0,1\}^n$.
\item For every constant $\epsilon>0$, there is no efficient algorithm that learns depth-$2$ neural networks with $n^\epsilon$ hidden neurons, on a distribution where each component is drawn i.i.d. from a (non-uniform) Bernoulli distribution.
\item For every constant $\epsilon>0$, there is no efficient algorithm that learns depth-$3$ neural networks with $n^\epsilon$ hidden neurons, on the uniform distribution over $\{0,1\}^n$.
\end{enumerate}
\end{theorem}

\begin{theorem}
\label{thm:nn fixed k}
Let $\ch \subseteq \reals^{(\{0,1\}^n)}$ be the functions expressible by depth-$2$ neural networks with $k$ hidden neurons and no activation function in the output neuron, where $k$ is a constant independent of $n$.
Let $\cl$ be a learning algorithm, that for every $\ch$-realizable distribution, returns with probability at least $\frac{3}{4}$ a hypothesis with error at most $\frac{1}{10}$.
Then, under Assumption~\ref{ass:xormaj}, there is a universal constant $\beta>0$ (independent of $k,n$) such that the time-complexity of $\cl$ is $\Omega(n^{\beta k})$.
\end{theorem}

\begin{remark}
\cite{applebaum2016algebraic} conjectured that Assumption~\ref{ass:xormaj} holds for $\alpha \geq 5$. It implies that Theorem~\ref{thm:nn fixed k} holds for, e.g., $\beta = \frac{1}{21}$.
\end{remark}

We now consider continuous input distributions. We focus here on the normal distribution, but our result can be extended to other continuous distributions (see proof in Appendix~\ref{app:proof nn normal}).

\begin{theorem}
\label{thm:nn normal}
Under Assumption~\ref{ass:localPRG}, for every constant $\epsilon>0$, there is no efficient algorithm that learns depth-$3$ neural networks with $n^\epsilon$ hidden neurons on the standard Gaussian distribution.
\end{theorem}

\subsection{Automata}

We show hardness of weakly-learning DFAs on the uniform distribution (see proof in Appendix~\ref{app:proof DFA}).

\begin{theorem}
\label{thm:DFA}
Under Assumption~\ref{ass:localPRG}, for every constants $c,\epsilon>0$, there is no efficient algorithm that $\frac{1}{n^c}$-weakly learns DFAs of size $n^\epsilon$, on the uniform distribution over $\{0,1\}^n$.
\end{theorem}

\subsection{Other classes}

Our results imply lower bounds for some additional classes.
We start with hardness of learning $\omega(1)$-sparse polynomial threshold functions on $\{0,1\}^n$. Recall that a $q$-sparse polynomial has at most $q$ monomials with non-zero coefficients.

\begin{corollary}
\label{cor:sparse poly}
Under Assumption~\ref{ass:localPRG}, for every $q(n)=\omega(1)$, there is no efficient algorithm that learns $q(n)$-sparse polynomial threshold functions over $\{0,1\}^n$.
\end{corollary}

Corollary~\ref{cor:sparse poly} follows from Theorem~\ref{thm:DNF} since any function realized by a DNF formula with $q(n)$ terms can be also realized by a polynomial threshold function over $\{0,1\}^n$ with $q(n)$ monomials.
We also consider $\omega(1)$-sparse $GF(2)$ polynomials over $\{0,1\}^n$. Such a polynomial is simply a sum modulo $2$ 
of $\omega(1)$ monomials (see proof in Appendix~\ref{app:proof sparse poly GF2}).
\begin{theorem}
\label{thm:sparse poly GF2}
Under Assumption~\ref{ass:localPRG}, for every $q(n)=\omega(1)$, there is no efficient algorithm that learns $q(n)$-sparse $GF(2)$ polynomials over $\{0,1\}^n$.
\end{theorem}

Finally, the following corollaries follow from the hardness of learning DNFs (see \cite{daniely2016complexity}).
We note that these results are already known under other assumptions \citep{feldman2006new,daniely2016half,blum2003noise,daniely2016complexity}.

\begin{corollary}
\label{cor:agnostic conjunctions}
Under Assumption~\ref{ass:localPRG}, there is no efficient algorithm that agnostically learns conjunctions.
\end{corollary}

\begin{corollary}
\label{cor:agnostic halfspaces}
Under Assumption~\ref{ass:localPRG}, there is no efficient algorithm that agnostically learns halfspaces.
\end{corollary}

\begin{corollary}
\label{cor:agnostic parities}
Under Assumption~\ref{ass:localPRG}, there is no efficient algorithm that agnostically learns parities.
\end{corollary}

\section{Our technique}
\label{sec:technique}

\stam{
In this section we describe the main method used in the proofs. 
The detailed proofs are in the appendix. 
}

\subsection{Hardness under Assumption~\ref{ass:localPRG}}
	
We first describe the proof ideas for the case of DNFs. Then, we explain how to apply the method to other classes.

\subsubsection{Distribution-free hardness for DNFs}

We describe the main ideas in the proof of the first part of Theorem~\ref{thm:DNF}.
We encode a hyperedge $S = (i_1,\ldots,i_k)$ by $\bz^S \in \{0,1\}^{kn}$, where $\bz^S$ is the concatenation of $k$ vectors in $\{0,1\}^n$, such that the $j$-th vector has $0$ in the $i_j$-th component and $1$ elsewhere. 
Thus, $\bz^S$ consists of $k$ size-$n$ slices, each encodes a member of $S$.
For a predicate $P:\{0,1\}^k \rightarrow \{0,1\}$ and $\bx \in \{0,1\}^n$, let $P_\bx:\{0,1\}^{kn} \rightarrow \{0,1\}$ be a function such that for every hyperedge $S$ we have $P_\bx(\bz^S) = P(\bx_{S})$.

Let $s>1$ be a constant. By Assumption~\ref{ass:localPRG}, there exists a constant $k$ and a predicate $P:\{0,1\}^k \rightarrow \{0,1\}$, such that $\cf_{P,n,n^s}$ is $\frac{1}{3}$-PRG.
Assume that there is an efficient algorithm $\cl$ that learns DNF formulas with $n'$ variables and $q(n')=\omega_{n'}(1)$ terms. 
We will use the algorithm $\cl$ to obtain an algorithm $\ca$ with distinguishing advantage greater than $\frac{1}{3}$ and thus reach a contradiction.

Given a sequence $(S_1,y_1),\ldots,(S_{n^s},y_{n^s})$, where $S_1,\ldots,S_{n^s}$ are i.i.d. random hyperedges, the algorithm $\ca$ needs to distinguish whether $\by = (y_1,\ldots,y_{n^s})$ is random or that $\by = (P(\bx_{S_1}),\ldots,P(\bx_{S_{n^s}})) = (P_\bx(\bz^{S_1}),\ldots,P_\bx(\bz^{S_{n^s}}))$ for a random $\bx \in \{0,1\}^n$.
Let $\cs = ((\bz^{S_1},y_1),\ldots,(\bz^{S_{n^s}},y_{n^s}))$.

We show that for every predicate $P:\{0,1\}^k \rightarrow \{0,1\}$ and $\bx \in \{0,1\}^n$, there is a DNF formula $\psi$ over $\{0,1\}^{kn}$ with at most $2^k$ terms, such that for every hyperedge $S$ we have $P_\bx(\bz^S)=\psi(\bz^S)$. 
The formula $\psi$ is such that for each satisfying assignment $\bb \in \{0,1\}^k$ of $P$ there is a term in $\psi$ that checks whether $\bx_S = \bb$. Thus, $\psi(\bz^S)=P(\bx_S)=P_\bx(\bz^{S})$.
Therefore, if $\cs$ is pseudorandom then it is realizable by a DNF formula with at most $2^k$ terms. 
Since $k$ is constant, then for a sufficiently large $n$ we have $2^k \leq q(kn)$.
Hence, the algorithm $\ca$ can distinguish whether $\cs$ is pseudorandom or random as follows. It partitions $\cs$ to a training set and a test set, and runs $\cl$ on the training set (we show that if $s$ is a sufficiently large constant then we can choose a training set large enough for $\cl$). Let $h$ be the hypothesis returned by $\cl$. If $\cs$ is pseudorandom then we show that $h$ will have small error on the test set, and if $\cs$ is random then $h$ will have large error on the test set. Hence, $\ca$ can distinguish between the cases.

\subsubsection{Distribution-specific hardness for DNFs}

We turn to describe the main ideas in the proof of the second part of Theorem~\ref{thm:DNF}. We show how to distinguish whether the sequence $\cs$ from the previous paragraph is random or pseudorandom, given access to a distribution-specific learning algorithm.
Let $\cl'$ be an efficient algorithm that learns DNF formulas with $n'$ variables and at most $(n')^\epsilon$ terms, on a distribution $\cd'$ such that each component is drawn i.i.d. from a Bernoulli distribution where the probability of $1$ is $p$. Assume that $p$ is such that the probability that a random $\bz \sim \cd'$ is an encoding of a hyperedge is not too small. 

We show an algorithm $\ca'$ such that given a sequence $\cs = ((\bz^{S_1},y_1),\ldots,(\bz^{S_{n^s}},y_{n^s}))$, it distinguishes whether $\cs$ is pseudorandom or random. Here, the algorithm $\ca'$ has access to $\cl'$, which is guaranteed to learn successfully only if the input distribution is $\cd'$. Note that for every $i \in [n^s]$ the vector $\bz^{S_i}$ is an encoding of a random hyperedge, and does not have the distribution $\cd'$. Therefore, the algorithm $\ca'$ will run $\cl'$ with an examples oracle that essentially works as follows: In the $i$-th call to the oracle, it chooses $\bz_i \sim \cd'$. If $\bz_i$ is an encoding of a hyperedge then the oracle returns $(\bz^{S_i},y_i)$, and otherwise it returns $(\bz_i,1)$. Namely, if $\bz_i$ is an encoding of a hyperedge then we replace it by $\bz^{S_i}$, which is an encoding of a random hyperedge, and hence we do not change the distribution. Thus, the oracle uses $\cs$ as a source for random encodings.

Let $h'$ be the hypothesis returned by $\cl'$. The algorithm $\ca'$ now checks $h'$ on a test set created by the examples oracle (we show that if $s$ is large enough then we can create sufficiently large training and test sets). If $\cs$ is pseudorandom then we show that the examples returned by the oracle are realized by some DNF formula with an appropriate number of terms, and hence $h'$ will have small error on the test set. 
Note that this DNF formula needs to return $1$ if the input is not an encoding of a hyperedge, and to return $P_\bx(\bz^S)$ if the input is the encoding $\bz^S$ of a hyperedge $S$.
If $\cs$ is random then $h'$ will be incorrect in roughly half of the examples in the test set that correspond to pairs $(\bz^{S_i},y_i)$ from $\cs$, and hence will have larger error on the test set. Therefore, $\ca'$ can distinguish between the cases.
	
\subsubsection{Distribution-specific hardness for other classes}
	
While each of the distribution-specific hardness results involves some unique challenges, all the proofs roughly follow a similar method to the one used in the case of DNFs. 

Let $\ch$ be the hypothesis class for which we want to show hardness and let $\cd$ be the input distribution.
Assuming that there is an algorithm $\cl$ that learns (or weakly learns) $\ch$ on the distribution $\cd$, we show an algorithm $\ca$ that distinguishes whether a sequence $\cs = ((S_1,y_1),\ldots,(S_{n^s},y_{n^s}))$ is pseudorandom or random. The algorithm $\ca$ runs $\cl$ with an examples oracle that can be implemented efficiently, and returns examples $(\bz,y)$ such that $\bz \sim \cd$. The oracle uses $\cs$ as a source for labeled random hyperedges, and with sufficiently high probability the returned example $(\bz,y)$ corresponds to some $(S_i,y_i)$ in $\cs$.
Let $h$ be the hypothesis returned by $\cl$. If $\cs$ is pseudorandom, then we show that the examples returned by the oracle are realizable by $\ch$, and hence $h$ has a small error on a test set created by the oracle. If $\cs$ is random then $h$ is incorrect in roughly half of the examples in the test set that correspond to pairs $(S_i,y_i)$ from $\cs$, and hence has larger error on the test set.
Hence, $\ca$ can distinguish between the cases. 

The implementation details of the above method are different for every class $\ch$ that we consider. Thus, in each proof we use different encodings of hyperedges and a different examples oracle. Moreover, in each proof we need to show that the examples returned by the oracle are realizable, and hence we construct a function $h \in \ch$ that labels correctly all examples returned by the oracle.

\subsection{Lower bounds under Assumption~\ref{ass:xormaj}}

We explain how to apply Assumption~\ref{ass:xormaj} in the case intersections of a constant number of halfspaces (we sketch here a proof for Theorem~\ref{thm:intersections}). The case of neural networks with a constant number of neurons (Theorem~\ref{thm:nn fixed k}) is similar.

It is not hard to show that Assumption~\ref{ass:xormaj} implies that there is a constant $\beta>0$ such that for every constant $k$, there is $l$ such that for the predicate $P=\xormaj_{k,l}$ the collection $\cf_{P,n,n^{2.1 \beta k}}$ is $\frac{1}{3}$-PRG.
Assume that there is an efficient algorithm $\cl$ that learns intersections of $k$ halfspaces over $\{0,1\}^{\tn}$. Assume that $\cl$ uses a sample of size $m(\tn) = \tn^{\beta k}$ and returns with probability at least $\frac{3}{4}$ a hypothesis with error at most $\frac{1}{10}$.
We will use the algorithm $\cl$ to establish an algorithm $\ca$ with distinguishing advantage greater than $\frac{1}{3}$ and thus reach a contradiction. It implies that an efficient algorithm that learns intersections of $k$ halfspaces over $\{0,1\}^{\tn}$ must use a sample of size greater than $\tn^{\beta k}$, and therefore runs in time $\Omega(\tn^{\beta k})$.

Let $\tn = \frac{(2n)(2n-1)}{2}+2n+1$.
For $\bz \in \{0,1\}^{2n}$, we denote by $\tilde{\bz} \in \{0,1\}^{\tn}$ the vector of all monomials over $\bz$ of degree at most $2$. We call $\tilde{\bz}$ the {\em monomials encoding} of $\bz$.
We encode a hyperedge $S = (i_1,\ldots,i_{k+l})$ by $\bz^S \in \{0,1\}^{2n}$, where $\bz^S$ is the concatenation of $2$ vectors in $\{0,1\}^n$, such that the first vector has $1$-bits in the indices $i_1,\ldots,i_k$ and $0$ elsewhere, and the second vector has $1$-bits in the indices $i_{k+1},\ldots,i_{k+l}$ and $0$ elsewhere. We denote by $\tilde{\bz}^S$ the monomials encoding of $\bz^S$.
For $\bx \in \{0,1\}^n$, let $P_\bx:\{0,1\}^{\tn} \rightarrow \{0,1\}$ be a function such that for every hyperedge $S$ we have $P_\bx(\tilde{\bz}^S) = P(\bx_{S})$.

We show that for every $\bx \in \{0,1\}^n$, there is a function $g:\{0,1\}^{\tn} \rightarrow \{0,1\}$ that can be expressed by an intersection of $k$ halfspaces, such that for every hyperedge $S$  we have $g(\tilde{\bz}^S)=P_\bx(\tilde{\bz}^S)$.
Intuitively, an intersection of $k$ halfspaces over $\tilde{\bz}^S$ is an intersection of $k$ degree-$2$ polynomial threshold functions over $\bz^S$, and we show that each degree-$2$ polynomial threshold function is powerful enough to handle the case where the number of $1$-bits in the XOR part of $P$ is $i$, for some $i \in [k]$.
With this claim at hand, we establish the algorithm $\ca$ as follows. 

Given a sequence $(S_1,y_1),\ldots,(S_{n^{2.1\beta k}},y_{n^{2.1\beta k}})$, where $S_1,\ldots,S_{n^{2.1\beta k}}$ are i.i.d. random hyperedges, the algorithm $\ca$ needs to distinguish whether $\by = (y_1,\ldots,y_{n^{2.1\beta k}})$ is random, or that $\by = (P(\bx_{S_1}),\ldots,P(\bx_{S_{n^{2.1\beta k}}})) = (P_\bx(\tilde{\bz}^{S_1}),\ldots,P_\bx(\tilde{\bz}^{S_{n^{2.1\beta k}}}))$ for a random $\bx \in \{0,1\}^n$.
Let $\cs = ((\tilde{\bz}^{S_1},y_1),\ldots,(\tilde{\bz}^{S_{n^{2.1\beta k}}},y_{n^{2.1\beta k}}))$.
The algorithm $\ca$ learns a function $h:\{0,1\}^{\tn} \rightarrow \{0,1\}$ by running $\cl$ with an examples oracle that in each call returns the next example from $\cs$. Recall that $\cl$ uses at most $m(\tn)=\tn^{\beta k}$ examples, and hence $\cs$ contains at least
\[
n^{2.1\beta k}-\tn^{\beta k}
\geq n^{2.1 \beta k}-(2n)^{2\beta k}
\geq n^{2.1 \beta k}-(n^{1.01})^{2\beta k}
= n^{2.1 \beta k}-n^{2.02\beta k}
= n^{2.02 \beta k}(n^{0.08 \beta k} - 1)
\geq \ln(n)
\]
examples that $\cl$ cannot view (for a sufficiently large $n$). We use these examples as a test set.
If $\cs$ is pseudorandom then it is realizable by an intersection of $k$ halfspaces, and thus w.h.p. $h$ has small error on the test set. If $\cs$ is random then $h$ has error of roughly $\frac{1}{2}$ on the test set. Hence, $\ca$ can distinguish between the cases.

\subsection*{Acknowledgements}

We thank Benny Applebaum and anonymous reviewers for their valuable comments.
This research is partially supported by ISF grant 2258/19.

\bibliographystyle{abbrvnat}
\bibliography{bib}

\begin{thebibliography}{68}
\providecommand{\natexlab}[1]{#1}
\providecommand{\url}[1]{\texttt{#1}}
\expandafter\ifx\csname urlstyle\endcsname\relax
  \providecommand{\doi}[1]{doi: #1}\else
  \providecommand{\doi}{doi: \begingroup \urlstyle{rm}\Url}\fi

\bibitem[Agarwal et~al.(2020)Agarwal, Awasthi, and Kale]{agarwal2020deep}
N.~Agarwal, P.~Awasthi, and S.~Kale.
\newblock A deep conditioning treatment of neural networks.
\newblock \emph{arXiv preprint arXiv:2002.01523}, 2020.

\bibitem[Angluin(1987)]{angluin1987learning}
D.~Angluin.
\newblock Learning regular sets from queries and counterexamples.
\newblock \emph{Information and computation}, 75\penalty0 (2):\penalty0
  87--106, 1987.

\bibitem[Applebaum(2013)]{applebaum2013pseudorandom}
B.~Applebaum.
\newblock Pseudorandom generators with long stretch and low locality from
  random local one-way functions.
\newblock \emph{SIAM Journal on Computing}, 42\penalty0 (5):\penalty0
  2008--2037, 2013.

\bibitem[Applebaum(2016)]{applebaum2016cryptographic}
B.~Applebaum.
\newblock Cryptographic hardness of random local functions.
\newblock \emph{Computational complexity}, 25\penalty0 (3):\penalty0 667--722,
  2016.

\bibitem[Applebaum and Lovett(2016)]{applebaum2016algebraic}
B.~Applebaum and S.~Lovett.
\newblock Algebraic attacks against random local functions and their
  countermeasures.
\newblock In \emph{Proceedings of the forty-eighth annual ACM symposium on
  Theory of Computing}, pages 1087--1100, 2016.

\bibitem[Applebaum and Raykov(2016)]{applebaum2016fast}
B.~Applebaum and P.~Raykov.
\newblock Fast pseudorandom functions based on expander graphs.
\newblock In \emph{Theory of Cryptography Conference}, pages 27--56. Springer,
  2016.

\bibitem[Applebaum et~al.(2008)Applebaum, Barak, and Xiao]{ApplebaumBaXi08}
B.~Applebaum, B.~Barak, and D.~Xiao.
\newblock On basing lower-bounds for learning on worst-case assumptions.
\newblock In \emph{Foundations of Computer Science, 2008. FOCS'08. IEEE 49th
  Annual IEEE Symposium on}, pages 211--220. IEEE, 2008.

\bibitem[Applebaum et~al.(2010)Applebaum, Barak, and
  Wigderson]{applebaum2010public}
B.~Applebaum, B.~Barak, and A.~Wigderson.
\newblock Public-key cryptography from different assumptions.
\newblock In \emph{Proceedings of the forty-second ACM symposium on Theory of
  computing}, pages 171--180, 2010.

\bibitem[Applebaum et~al.(2017)Applebaum, Damg{\aa}rd, Ishai, Nielsen, and
  Zichron]{applebaum2017secure}
B.~Applebaum, I.~Damg{\aa}rd, Y.~Ishai, M.~Nielsen, and L.~Zichron.
\newblock Secure arithmetic computation with constant computational overhead.
\newblock In \emph{Annual International Cryptology Conference}, pages 223--254.
  Springer, 2017.

\bibitem[Arora et~al.(2014)Arora, Bhaskara, Ge, and Ma]{arora2014provable}
S.~Arora, A.~Bhaskara, R.~Ge, and T.~Ma.
\newblock Provable bounds for learning some deep representations.
\newblock In \emph{International Conference on Machine Learning}, pages
  584--592, 2014.

\bibitem[Bakshi et~al.(2019)Bakshi, Jayaram, and Woodruff]{bakshi2019learning}
A.~Bakshi, R.~Jayaram, and D.~P. Woodruff.
\newblock Learning two layer rectified neural networks in polynomial time.
\newblock In \emph{Conference on Learning Theory}, pages 195--268. PMLR, 2019.

\bibitem[Baum(1990)]{baum1990polynomial}
E.~B. Baum.
\newblock A polynomial time algorithm that learns two hidden unit nets.
\newblock \emph{Neural Computation}, 2\penalty0 (4):\penalty0 510--522, 1990.

\bibitem[Blais et~al.(2010)Blais, O’Donnell, and Wimmer]{blais2010polynomial}
E.~Blais, R.~O’Donnell, and K.~Wimmer.
\newblock Polynomial regression under arbitrary product distributions.
\newblock \emph{Machine learning}, 80\penalty0 (2-3):\penalty0 273--294, 2010.

\bibitem[Blum et~al.(2003)Blum, Kalai, and Wasserman]{blum2003noise}
A.~Blum, A.~Kalai, and H.~Wasserman.
\newblock Noise-tolerant learning, the parity problem, and the statistical
  query model.
\newblock \emph{Journal of the ACM (JACM)}, 50\penalty0 (4):\penalty0 506--519,
  2003.

\bibitem[Blum and Kannan(1997)]{blum1997learning}
A.~L. Blum and R.~Kannan.
\newblock Learning an intersection of a constant number of halfspaces over a
  uniform distribution.
\newblock \emph{Journal of Computer and System Sciences}, 54\penalty0
  (2):\penalty0 371--380, 1997.

\bibitem[Boppana(1997)]{boppana1997average}
R.~B. Boppana.
\newblock The average sensitivity of bounded-depth circuits.
\newblock \emph{Information processing letters}, 63\penalty0 (5):\penalty0
  257--261, 1997.

\bibitem[Brutzkus and Globerson(2017)]{brutzkus2017globally}
A.~Brutzkus and A.~Globerson.
\newblock Globally optimal gradient descent for a convnet with gaussian inputs.
\newblock In \emph{Proceedings of the 34th International Conference on Machine
  Learning-Volume 70}, pages 605--614. JMLR. org, 2017.

\bibitem[Couteau et~al.(2018)Couteau, Dupin, M{\'e}aux, Rossi, and
  Rotella]{couteau2018concrete}
G.~Couteau, A.~Dupin, P.~M{\'e}aux, M.~Rossi, and Y.~Rotella.
\newblock On the concrete security of goldreich’s pseudorandom generator.
\newblock In \emph{International Conference on the Theory and Application of
  Cryptology and Information Security}, pages 96--124. Springer, 2018.

\bibitem[Daniely(2016)]{daniely2016half}
A.~Daniely.
\newblock Complexity theoretic limitations on learning halfspaces.
\newblock In \emph{Proceedings of the forty-eighth annual ACM symposium on
  Theory of Computing}, pages 105--117. ACM, 2016.

\bibitem[Daniely and Shalev-Shwartz(2016)]{daniely2016complexity}
A.~Daniely and S.~Shalev-Shwartz.
\newblock Complexity theoretic limitations on learning dnf’s.
\newblock In \emph{Conference on Learning Theory}, pages 815--830, 2016.

\bibitem[Daniely and Vardi(2020)]{daniely2020hardness}
A.~Daniely and G.~Vardi.
\newblock Hardness of learning neural networks with natural weights.
\newblock \emph{arXiv preprint arXiv:2006.03177}, 2020.

\bibitem[Daniely et~al.(2014)Daniely, Linial, and
  Shalev-Shwartz]{daniely2013average}
A.~Daniely, N.~Linial, and S.~Shalev-Shwartz.
\newblock From average case complexity to improper learning complexity.
\newblock In \emph{STOC}, 2014.

\bibitem[Das et~al.(2019)Das, Gollapudi, Kumar, and
  Panigrahy]{das2019learnability}
A.~Das, S.~Gollapudi, R.~Kumar, and R.~Panigrahy.
\newblock On the learnability of deep random networks.
\newblock \emph{arXiv preprint arXiv:1904.03866}, 2019.

\bibitem[Diakonikolas et~al.(2020{\natexlab{a}})Diakonikolas, Kane, and
  Zarifis]{diakonikolas2020near}
I.~Diakonikolas, D.~Kane, and N.~Zarifis.
\newblock Near-optimal sq lower bounds for agnostically learning halfspaces and
  relus under gaussian marginals.
\newblock \emph{Advances in Neural Information Processing Systems}, 33,
  2020{\natexlab{a}}.

\bibitem[Diakonikolas et~al.(2020{\natexlab{b}})Diakonikolas, Kane, Kontonis,
  and Zarifis]{diakonikolas2020algorithms}
I.~Diakonikolas, D.~M. Kane, V.~Kontonis, and N.~Zarifis.
\newblock Algorithms and sq lower bounds for pac learning one-hidden-layer relu
  networks.
\newblock \emph{arXiv preprint arXiv:2006.12476}, 2020{\natexlab{b}}.

\bibitem[Du and Goel(2018)]{du2018improved}
S.~S. Du and S.~Goel.
\newblock Improved learning of one-hidden-layer convolutional neural networks
  with overlaps.
\newblock \emph{arXiv preprint arXiv:1805.07798}, 2018.

\bibitem[Du et~al.(2017{\natexlab{a}})Du, Lee, and Tian]{du2017convolutional}
S.~S. Du, J.~D. Lee, and Y.~Tian.
\newblock When is a convolutional filter easy to learn?
\newblock \emph{arXiv preprint arXiv:1709.06129}, 2017{\natexlab{a}}.

\bibitem[Du et~al.(2017{\natexlab{b}})Du, Lee, Tian, Poczos, and
  Singh]{du2017gradient}
S.~S. Du, J.~D. Lee, Y.~Tian, B.~Poczos, and A.~Singh.
\newblock Gradient descent learns one-hidden-layer cnn: Don't be afraid of
  spurious local minima.
\newblock \emph{arXiv preprint arXiv:1712.00779}, 2017{\natexlab{b}}.

\bibitem[Feldman et~al.(2006)Feldman, Gopalan, Khot, and
  Ponnuswami]{feldman2006new}
V.~Feldman, P.~Gopalan, S.~Khot, and A.~K. Ponnuswami.
\newblock New results for learning noisy parities and halfspaces.
\newblock In \emph{2006 47th Annual IEEE Symposium on Foundations of Computer
  Science (FOCS'06)}, pages 563--574. IEEE, 2006.

\bibitem[Feldman et~al.(2015)Feldman, Perkins, and Vempala]{FeldmanPeVe2015}
V.~Feldman, W.~Perkins, and S.~Vempala.
\newblock On the complexity of random satisfiability problems with planted
  solutions.
\newblock In \emph{STOC}, 2015.

\bibitem[Fish and Reyzin(2017)]{fish2017open}
B.~Fish and L.~Reyzin.
\newblock Open problem: Meeting times for learning random automata.
\newblock In \emph{Conference on Learning Theory}, pages 8--11, 2017.

\bibitem[Freund(1995)]{Freund95}
Y.~Freund.
\newblock Boosting a weak learning algorithm by majority.
\newblock \emph{Information and Computation}, 121\penalty0 (2):\penalty0
  256--285, 1995.

\bibitem[Furst et~al.(1991)Furst, Jackson, and Smith]{furst1991improved}
M.~L. Furst, J.~C. Jackson, and S.~W. Smith.
\newblock Improved learning of ac0 functions.
\newblock In \emph{COLT}, volume~91, pages 317--325, 1991.

\bibitem[Goel and Klivans(2017)]{goel2017learning}
S.~Goel and A.~Klivans.
\newblock Learning neural networks with two nonlinear layers in polynomial
  time.
\newblock \emph{arXiv preprint arXiv:1709.06010}, 2017.

\bibitem[Goel et~al.(2018)Goel, Klivans, and Meka]{goel2018learning}
S.~Goel, A.~Klivans, and R.~Meka.
\newblock Learning one convolutional layer with overlapping patches.
\newblock \emph{arXiv preprint arXiv:1802.02547}, 2018.

\bibitem[Goel et~al.(2019)Goel, Karmalkar, and Klivans]{goel2019time}
S.~Goel, S.~Karmalkar, and A.~Klivans.
\newblock Time/accuracy tradeoffs for learning a relu with respect to gaussian
  marginals.
\newblock In \emph{Advances in Neural Information Processing Systems}, pages
  8584--8593, 2019.

\bibitem[Goel et~al.(2020{\natexlab{a}})Goel, Gollakota, Jin, Karmalkar, and
  Klivans]{goel2020superpolynomial}
S.~Goel, A.~Gollakota, Z.~Jin, S.~Karmalkar, and A.~Klivans.
\newblock Superpolynomial lower bounds for learning one-layer neural networks
  using gradient descent.
\newblock \emph{arXiv preprint arXiv:2006.12011}, 2020{\natexlab{a}}.

\bibitem[Goel et~al.(2020{\natexlab{b}})Goel, Gollakota, and
  Klivans]{goel2020statistical}
S.~Goel, A.~Gollakota, and A.~Klivans.
\newblock Statistical-query lower bounds via functional gradients.
\newblock \emph{Advances in Neural Information Processing Systems}, 33,
  2020{\natexlab{b}}.

\bibitem[Goel et~al.(2020{\natexlab{c}})Goel, Klivans, Manurangsi, and
  Reichman]{goel2020tight}
S.~Goel, A.~Klivans, P.~Manurangsi, and D.~Reichman.
\newblock Tight hardness results for training depth-2 relu networks.
\newblock \emph{arXiv preprint arXiv:2011.13550}, 2020{\natexlab{c}}.

\bibitem[Goldreich(2000)]{goldreich2000candidate}
O.~Goldreich.
\newblock Candidate one-way functions based on expander graphs.
\newblock \emph{IACR Cryptol. ePrint Arch.}, 2000:\penalty0 63, 2000.

\bibitem[H{\aa}stad(2001)]{haastad2001slight}
J.~H{\aa}stad.
\newblock A slight sharpening of lmn.
\newblock \emph{Journal of Computer and System Sciences}, 63\penalty0
  (3):\penalty0 498--508, 2001.

\bibitem[Hellerstein and Servedio(2007)]{hellerstein2007pac}
L.~Hellerstein and R.~A. Servedio.
\newblock On pac learning algorithms for rich boolean function classes.
\newblock \emph{Theoretical Computer Science}, 384\penalty0 (1):\penalty0
  66--76, 2007.

\bibitem[Ishai et~al.(2008)Ishai, Kushilevitz, Ostrovsky, and
  Sahai]{ishai2008cryptography}
Y.~Ishai, E.~Kushilevitz, R.~Ostrovsky, and A.~Sahai.
\newblock Cryptography with constant computational overhead.
\newblock In \emph{Proceedings of the fortieth annual ACM symposium on Theory
  of computing}, pages 433--442, 2008.

\bibitem[Janzamin et~al.(2015)Janzamin, Sedghi, and
  Anandkumar]{janzamin2015beating}
M.~Janzamin, H.~Sedghi, and A.~Anandkumar.
\newblock Beating the perils of non-convexity: Guaranteed training of neural
  networks using tensor methods.
\newblock \emph{arXiv preprint arXiv:1506.08473}, 2015.

\bibitem[Kearns and Valiant(1994)]{KearnsVa94}
M.~Kearns and L.~G. Valiant.
\newblock Cryptographic limitations on learning {B}oolean formulae and finite
  automata.
\newblock \emph{Journal of the Association for Computing Machinery},
  41\penalty0 (1):\penalty0 67--95, Jan. 1994.

\bibitem[Kharitonov(1993)]{Kharitonov93}
M.~Kharitonov.
\newblock Cryptographic hardness of distribution-specific learning.
\newblock In \emph{Proceedings of the twenty-fifth annual ACM symposium on
  Theory of computing}, pages 372--381. ACM, 1993.

\bibitem[Klivans and Servedio(2001)]{klivans2001learning}
A.~R. Klivans and R.~Servedio.
\newblock Learning dnf in time $2^{O(n^{1/3})}$.
\newblock In \emph{Proceedings of the thirty-third annual ACM symposium on
  Theory of computing}, pages 258--265. ACM, 2001.

\bibitem[Klivans and Sherstov(2006)]{KlivansSh06}
A.~R. Klivans and A.~A. Sherstov.
\newblock Cryptographic hardness for learning intersections of halfspaces.
\newblock In \emph{FOCS}, 2006.

\bibitem[Klivans et~al.(2004)Klivans, O'Donnell, and
  Servedio]{klivans2004learning}
A.~R. Klivans, R.~O'Donnell, and R.~A. Servedio.
\newblock Learning intersections and thresholds of halfspaces.
\newblock \emph{Journal of Computer and System Sciences}, 68\penalty0
  (4):\penalty0 808--840, 2004.

\bibitem[Klivans et~al.(2009)Klivans, Long, and Tang]{klivans2009baum}
A.~R. Klivans, P.~M. Long, and A.~K. Tang.
\newblock Baum’s algorithm learns intersections of halfspaces with respect to
  log-concave distributions.
\newblock In \emph{Approximation, Randomization, and Combinatorial
  Optimization. Algorithms and Techniques}, pages 588--600. Springer, 2009.

\bibitem[Kothari and Livni(2018)]{kothari2018improper}
P.~K. Kothari and R.~Livni.
\newblock Improper learning by refuting.
\newblock In \emph{9th Innovations in Theoretical Computer Science Conference
  (ITCS 2018)}. Schloss Dagstuhl-Leibniz-Zentrum fuer Informatik, 2018.

\bibitem[Li and Yuan(2017)]{li2017convergence}
Y.~Li and Y.~Yuan.
\newblock Convergence analysis of two-layer neural networks with relu
  activation.
\newblock In \emph{Advances in Neural Information Processing Systems}, pages
  597--607, 2017.

\bibitem[Lin(2016)]{lin2016indistinguishability}
H.~Lin.
\newblock Indistinguishability obfuscation from constant-degree graded encoding
  schemes.
\newblock In \emph{Annual International Conference on the Theory and
  Applications of Cryptographic Techniques}, pages 28--57. Springer, 2016.

\bibitem[Lin and Vaikuntanathan(2016)]{linVai2016indistinguishability}
H.~Lin and V.~Vaikuntanathan.
\newblock Indistinguishability obfuscation from ddh-like assumptions on
  constant-degree graded encodings.
\newblock In \emph{2016 IEEE 57th Annual Symposium on Foundations of Computer
  Science (FOCS)}, pages 11--20. IEEE, 2016.

\bibitem[Linial et~al.(1993)Linial, Mansour, and Nisan]{LinialMaNi93}
N.~Linial, Y.~Mansour, and N.~Nisan.
\newblock Constant depth circuits, {F}ourier transform, and learnability.
\newblock \emph{Journal of the Association for Computing Machinery},
  40\penalty0 (3):\penalty0 607--620, July 1993.

\bibitem[M{\'e}aux et~al.(2019)M{\'e}aux, Carlet, Journault, and
  Standaert]{meaux2019improved}
P.~M{\'e}aux, C.~Carlet, A.~Journault, and F.-X. Standaert.
\newblock Improved filter permutators: Combining symmetric encryption design,
  boolean functions, low complexity cryptography, and homomorphic encryption,
  for private delegation of computations.
\newblock \emph{IACR Cryptol. ePrint Arch.}, 2019:\penalty0 483, 2019.

\bibitem[Michaliszyn and Otop(2019)]{michaliszyn2019approximate}
J.~Michaliszyn and J.~Otop.
\newblock Approximate learning of limit-average automata.
\newblock \emph{arXiv preprint arXiv:1906.11104}, 2019.

\bibitem[Nanashima(2020)]{nanashima2020extending}
M.~Nanashima.
\newblock Extending learnability to auxiliary-input cryptographic primitives
  and meta-pac learning.
\newblock In \emph{Conference on Learning Theory}, pages 2998--3029. PMLR,
  2020.

\bibitem[O'Donnell and Witmer(2014)]{odonnell2014goldreich}
R.~O'Donnell and D.~Witmer.
\newblock Goldreich's prg: Evidence for near-optimal polynomial stretch.
\newblock In \emph{2014 IEEE 29th Conference on Computational Complexity
  (CCC)}, pages 1--12. IEEE, 2014.

\bibitem[Pitt(1989)]{pitt1989inductive}
L.~Pitt.
\newblock Inductive inference, dfas, and computational complexity.
\newblock In \emph{International Workshop on Analogical and Inductive
  Inference}, pages 18--44. Springer, 1989.

\bibitem[Schapire(1989)]{Schapire89}
R.~Schapire.
\newblock The strength of weak learnability.
\newblock In \emph{FOCS}, pages 28--33, Oct. 1989.

\bibitem[Shamir(2018)]{shamir2018distribution}
O.~Shamir.
\newblock Distribution-specific hardness of learning neural networks.
\newblock \emph{The Journal of Machine Learning Research}, 19\penalty0
  (1):\penalty0 1135--1163, 2018.

\bibitem[Song et~al.(2017)Song, Vempala, Wilmes, and Xie]{song2017complexity}
L.~Song, S.~Vempala, J.~Wilmes, and B.~Xie.
\newblock On the complexity of learning neural networks.
\newblock In \emph{Advances in neural information processing systems}, pages
  5514--5522, 2017.

\bibitem[Tian(2017)]{tian2017analytical}
Y.~Tian.
\newblock An analytical formula of population gradient for two-layered relu
  network and its applications in convergence and critical point analysis.
\newblock In \emph{Proceedings of the 34th International Conference on Machine
  Learning-Volume 70}, pages 3404--3413. JMLR. org, 2017.

\bibitem[Vadhan(2017)]{vadhan2017learning}
S.~Vadhan.
\newblock On learning vs. refutation.
\newblock In \emph{Conference on Learning Theory}, pages 1835--1848. PMLR,
  2017.

\bibitem[Valiant(1984)]{Valiant84}
L.~G. Valiant.
\newblock A theory of the learnable.
\newblock \emph{Communications of the ACM}, 27\penalty0 (11):\penalty0
  1134--1142, Nov. 1984.

\bibitem[Vempala(1997)]{vempala1997random}
S.~Vempala.
\newblock A random sampling based algorithm for learning the intersection of
  half-spaces.
\newblock In \emph{Proceedings 38th Annual Symposium on Foundations of Computer
  Science}, pages 508--513. IEEE, 1997.

\bibitem[Vempala and Wilmes(2019)]{vempala2019gradient}
S.~Vempala and J.~Wilmes.
\newblock Gradient descent for one-hidden-layer neural networks: Polynomial
  convergence and sq lower bounds.
\newblock In \emph{Conference on Learning Theory}, pages 3115--3117. PMLR,
  2019.

\end{thebibliography}

\appendix

\section{Proofs}

\subsection{Proof of Theorem~\ref{thm:DNF}}
\label{app:proof DNF}

We encode a hyperedge $S = (i_1,\ldots,i_k)$ by $\bz^S \in \{0,1\}^{kn}$, where $\bz^S$ is the concatenation of $k$ vectors in $\{0,1\}^n$, such that the $j$-th vector has $0$ in the $i_j$-th component and $1$ elsewhere. 
Thus, $\bz^S$ consists of $k$ size-$n$ slices, each encodes a member of $S$.
For $\bz \in \{0,1\}^{kn}$, $i \in [k]$ and $j \in [n]$, we denote $z_{i,j}=z_{(i-1) \cdot n + j}$. That is, $z_{i,j}$ is the $j$-th component in the $i$-th slice in $\bz$.
For a predicate $P:\{0,1\}^k \rightarrow \{0,1\}$ and $\bx \in \{0,1\}^n$, let $P_\bx:\{0,1\}^{kn} \rightarrow \{0,1\}$ be a function such that for every hyperedge $S$ we have $P_\bx(\bz^S) = P(\bx_{S})$.

\begin{lemma}
\label{lemma:from P to DNF}
For every predicate $P:\{0,1\}^k \rightarrow \{0,1\}$ and $\bx \in \{0,1\}^n$, there is a DNF formula $\psi$ over $\{0,1\}^{kn}$ with at most $2^k$ terms, such that for every hyperedge $S$ we have $P_\bx(\bz^S)=\psi(\bz^S)$.
\end{lemma}
\begin{proof}
We denote by $\cb \subseteq \{0,1\}^{k}$ the set of satisfying assignments of $P$. Note that the size of $\cb$ is at most $2^k$.
Consider the following DNF formula over $\{0,1\}^{kn}$:
\[
\psi(\bz)
= \bigvee_{\bb \in \cb} \bigwedge_{j \in [k]} \bigwedge_{\{l:x_l \neq b_j \}} z_{j,l}~.
\]
For a hyperedge $S=(i_1,\ldots,i_k)$, we have
\begin{align*}
\psi(\bz^S)=1
&\iff \exists \bb \in \cb \; \forall j \in [k] \; \forall x_l \neq b_j, \; z^S_{j,l}=1
\\
&\iff \exists \bb \in \cb \; \forall j \in [k] \; \forall x_l \neq b_j, \; i_j \neq l
\\
&\iff \exists \bb \in \cb \; \forall j \in [k], \; x_{i_j} = b_j
\\
&\iff \exists \bb \in \cb, \; \bx_S = \bb
\\
&\iff P(\bx_S)=1
\\
&\iff P_\bx(\bz^S)=1~.
\end{align*}
\end{proof}

In the following lemma we prove the first part of the theorem.

\begin{lemma}
\label{lemma:DNF distribution free}
Under Assumption~\ref{ass:localPRG}, there is no efficient algorithm that learns DNF formulas with $n$ variables and $\omega_n(1)$ terms.
\end{lemma}
\begin{proof}
Assume that there is an efficient algorithm $\cl$ that learns DNF formulas with $n'$ variables and $q(n')=\omega_{n'}(1)$ terms. Let $m(n')$ be a polynomial such that $\cl$ uses a sample of size at most $m(n')$ and returns with probability at least $\frac{3}{4}$ a hypothesis with error at most $\frac{1}{10}$.
Let $s>1$ be a constant such that $n^s \geq m(n\log(n))+n$ for every sufficiently large $n$.
By Assumption~\ref{ass:localPRG}, there exists a constant $k$ and a predicate $P:\{0,1\}^k \rightarrow \{0,1\}$, such that $\cf_{P,n,n^s}$ is $\frac{1}{3}$-PRG.
We will show an algorithm $\ca$ with distinguishing advantage greater than $\frac{1}{3}$ and thus reach a contradiction.

Given a sequence $(S_1,y_1),\ldots,(S_{n^s},y_{n^s})$, where $S_1,\ldots,S_{n^s}$ are i.i.d. random hyperedges, the algorithm $\ca$ needs to distinguish whether $\by = (y_1,\ldots,y_{n^s})$ is random or that $\by = (P(\bx_{S_1}),\ldots,P(\bx_{S_{n^s}})) = (P_\bx(\bz^{S_1}),\ldots,P_\bx(\bz^{S_{n^s}}))$ for a random $\bx \in \{0,1\}^n$.
Let $\cs = ((\bz^{S_1},y_1),\ldots,(\bz^{S_{n^s}},y_{n^s}))$.
Let $\cd$ be a distribution on $\{0,1\}^{kn}$ such that $\bz \sim \cd$ is an encoding of a random hyperedge. Note that each $\bz^{S_i}$ from $\cs$ is drawn i.i.d. from $\cd$.

We use the efficient algorithm $\cl$ in order to obtain distinguishing advantage greater than $\frac{1}{3}$ as follows.
The algorithm $\ca$ learns a hypothesis $h:\{0,1\}^{kn} \rightarrow \{0,1\}$ by running $\cl$ with an examples oracle that in each call returns the next example from $\cs$. Recall that $\cl$ uses at most $m(kn) \leq m(n\log(n))$ examples (assuming $n$ is large enough), and hence $\cs$ contains at least $n$ examples that $\cl$ cannot view. We denote the indices of these examples by $I = \{m(n\log(n))+1,\ldots,m(n\log(n))+n\}$, and the examples by $\cs_I = \{(\bz^{S_i},y_i)\}_{i \in I}$.
Let $\ell_{I}(h)=\frac{1}{|I|}\sum_{i \in I}\onefunc(h(\bz^{S_i}) \neq y_i)$.
Now, if $\ell_I(h) \leq \frac{2}{10}$, then $\ca$ returns $1$, and otherwise it returns $0$.

Clearly, the algorithm $\ca$ runs in polynomial time.
We now show that if $\cs$ is pseudorandom then $\ca$ returns $1$ with probability greater than $\frac{2}{3}$, and if $\cs$ is random then $\ca$ returns $1$ with probability less than $\frac{1}{3}$.
By Lemma~\ref{lemma:from P to DNF}, there is a DNF formula $\psi_\bx$ over $\{0,1\}^{kn}$ with at most $2^k < q(kn)$ terms (for a sufficiently large $n$), such that for every hyperedge $S$ we have $P_\bx(\bz^S)=\psi_\bx(\bz^S)$. Thus, $\cs$ is realized by $\psi_\bx$.
Hence, if $\cs$ is pseudorandom then with probability at least $\frac{3}{4}$ the algorithm $\cl$ returns a hypothesis $h$ such that $\E_{\bz \sim \cd} \onefunc (h(\bz) \neq P_\bx(\bz)) \leq \frac{1}{10}$. Therefore, $\E_{\cs_I}\ell_{I}(h) = \E_{\bz \sim \cd}\onefunc(h(\bz) \neq P_\bx(\bz)) \leq \frac{1}{10}$.
If $\cs$ is random then for every function $h:\{0,1\}^{kn} \rightarrow \{0,1\}$ the events $\{h(\bz^{S_i}) = y_i\}_{i \in I}$ are independent from one another, and each has probability $\frac{1}{2}$. Hence, $\E_{\cs_I}\ell_{I}(h) = \frac{1}{2}$.

By the Hoefding bound, for a sufficiently large $n$ we have
\[
\Pr_{\cs_I}\left[\left|\ell_{I}(h) -  \E_{\cs_I}\ell_{I}(h)\right| \geq \frac{1}{10} \right] \leq \frac{1}{20}~.
\]
Therefore, if $\cs$ is pseudorandom then for a sufficiently large $n$ we have with probability at least $1-\left(\frac{1}{4} + \frac{1}{20}\right) = \frac{7}{10} > \frac{2}{3}$ that $\E_{\cs_I}\ell_{I}(h) \leq \frac{1}{10}$ and $\left|\ell_{I}(h) -  \E_{\cs_I}\ell_{I}(h)\right| < \frac{1}{10}$, and hence $\ell_{I}(h) \leq \frac{2}{10}$. Thus, the algorithm $\ca$ returns $1$ with probability greater than $\frac{2}{3}$.
If $\cs$ is random then $\E_{\cs}\ell_{I}(h) = \frac{1}{2}$ and for a sufficiently large $n$ we have with probability at least $\frac{19}{20}$ that $\left|\ell_{I}(h) - \E_{\cs_I}\ell_{I}(h)\right| < \frac{1}{10}$. Hence, with probability greater than $\frac{2}{3}$ we have $\ell_{I}(h) > \frac{2}{10}$ and the algorithm $\ca$ returns $0$.
\end{proof}

We will use the following lemma throughout our proofs.

\begin{lemma}
\label{lemma:hoefding1}
Let $c \geq 0$ be a constant. Let $\xi_1,\ldots,\xi_{n^{2c+3}}$ be a sequence of i.i.d. random variables and let $\xi = \frac{1}{n^{2c+3}}\sum_{i \in [n^{2c+3}]}\xi_i$. Assume that $\Pr\left[0 \leq \xi_i \leq 1\right] = 1$ for every $i$. Then for a sufficiently large $n$ we have
\[
\Pr\left[\left|\xi -  \E[\xi]\right| \geq \frac{1}{n^{c+1}} \right] < \frac{1}{20}~.
\]
\end{lemma}
\begin{proof}
By the Hoefding bound we have
\begin{align*}
\Pr\left[\left|\xi - \E[\xi]\right| \geq \frac{1}{n^{c+1}} \right]
\leq 2\exp\left(-\frac{2n^{2c+3}}{n^{(c+1)\cdot 2}}\right)
= 2\exp\left(-2n\right)~.
\end{align*}
Thus, for a sufficiently large $n$ the requirement holds.
\end{proof}

In the following lemma we prove the second part of the theorem.

\begin{lemma}
\label{lemma:DNF distribution specific}
For every constant $\epsilon>0$, there is no efficient algorithm that learns DNF formulas with $n^\epsilon$ terms, on a distribution such that each component is drawn i.i.d. from a (non-uniform) Bernoulli distribution.
\end{lemma}
\begin{proof}
Consider the distribution $\cd$ over $\{0,1\}^{n^{1+3/\epsilon}}$, such that each component is drawn i.i.d. from a Bernoulli distribution where the probability of $0$ is $\frac{1}{n}$.
Assume that there is an efficient algorithm $\cl$ that learns DNF formulas over $\{0,1\}^{n^{1+3/\epsilon}}$ with at most $n^3$ terms on the distribution $\cd$.
Let $m(n)$ be a polynomial such that $\cl$ uses a sample of size at most $m(n)$ and returns with probability at least $\frac{3}{4}$ a hypothesis with error at most $\frac{1}{n}$.
Let $s>1$ be a constant such that $n^s \geq m(n)+n^3$ for every sufficiently large $n$.
By Assumption~\ref{ass:localPRG}, there exists a constant $k$ and a predicate $P:\{0,1\}^k \rightarrow \{0,1\}$, such that $\cf_{P,n,n^s}$ is $\frac{1}{3}$-PRG.
We will show an algorithm $\ca$ with distinguishing advantage greater than $\frac{1}{3}$ and thus reach a contradiction.

We say that $\bz \in \{0,1\}^{n^{1+3/\epsilon}}$ is an extended encoding of a hyperedge if $(z_1,\ldots,z_{kn}) = \bz^S$ for some hyperedge $S$. That is, in each of the first $k$ size-$n$ slices in $\bz$ there is exactly one $0$-bit and each two of the first $k$ slices in $\bz$ encode different indices. Assuming that $n^{1+3/\epsilon} \geq kn$, the probability that $\bz \sim \cd$ is an extended encoding of a hyperedge, is given by
\begin{align*}
n \cdot (n-1) \cdot \ldots \cdot(n-k+1) \cdot \left(\frac{1}{n}\right)^k \left(\frac{n-1}{n}\right)^{nk-k}
&\geq \left(\frac{n-k}{n}\right)^k \left(\frac{n-1}{n}\right)^{k(n-1)}
\\
&=\left(1-\frac{k}{n}\right)^k \left(1-\frac{1}{n}\right)^{k(n-1)}~.
\end{align*}
Since for every $x \in (0,1)$ we have $e^{-x} < 1 - \frac{x}{2}$, then for a sufficiently large $n$ the above is at least
\begin{equation}
\label{eq:dnf-random encoding is hyperedge}
\exp\left(-\frac{2k^2}{n}\right) \cdot \exp\left(-\frac{2k(n-1)}{n} \right)
\geq \exp\left(-1\right) \cdot \exp\left(-2k\right)
\geq \frac{1}{\log(n)}~.
\end{equation}

Given a sequence $(S_1,y_1),\ldots,(S_{n^s},y_{n^s})$, where $S_1,\ldots,S_{n^s}$ are i.i.d. random hyperedges, the algorithm $\ca$ needs to distinguish whether $\by = (y_1,\ldots,y_{n^s})$ is random or that $\by = (P(\bx_{S_1}),\ldots,P(\bx_{S_{n^s}})) = (P_\bx(\bz^{S_1}),\ldots,P_\bx(\bz^{S_{n^s}}))$ for a random $\bx \in \{0,1\}^n$.
Let $\cs = (\bz^{S_1},y_1),\ldots,(\bz^{S_{n^s}},y_{n^s})$.

We use the efficient algorithm $\cl$ in order to obtain distinguishing advantage greater than $\frac{1}{3}$ as follows.
The algorithm $\ca$ runs $\cl$ with the following examples oracle.
In the $i$-th call to the oracle, it chooses $\bz_i \in \{0,1\}^{n^{1+3/\epsilon}}$ according to $\cd$.
If $\bz_i$ is not an extended encoding of a hyperedge (with probability at most $1-\frac{1}{\log(n)}$ by Eq.~\ref{eq:dnf-random encoding is hyperedge}) then the oracle returns $(\bz'_i,y'_i)$ where $\bz'_i=\bz_i$ and $y'_i=1$. Otherwise, the oracle obtains a vector $\bz'_i$ by replacing the first $kn$ components in $\bz_i$ with $\bz^{S_i}$, and returns $(\bz'_i,y'_i)$ where $y'_i=y_i$. Note that the vector $\bz'_i$ returned by the oracle has the distribution $\cd$, since replacing a random hyperedge with another random hyperedge does not change the distribution. Let $h$ be the hypothesis returned by $\cl$.
Recall that $\cl$ uses at most $m(n)$ examples, and hence $\cs$ contains at least $n^3$ examples that $\cl$ cannot view. We denote the indices of these examples by $I = \{m(n)+1,\ldots,m(n)+n^3\}$, and the examples by $\cs_I = \{(\bz^{S_i},y_i)\}_{i \in I}$. By $n^3$ additional calls to the oracle, the algorithm $\ca$ obtains the examples $\cs'_I = \{(\bz'_i,y'_i)\}_{i \in I}$ that correspond to $\cs_I$.
Let $\ell_{I}(h)=\frac{1}{|I|}\sum_{i \in I}\onefunc(h(\bz'_i) \neq y'_i)$.
Now, if $\ell_I(h) \leq \frac{2}{n}$, then $\ca$ returns $1$, and otherwise it returns $0$.
Clearly, the algorithm $\ca$ runs in polynomial time.
We now show that if $\cs$ is pseudorandom then $\ca$ returns $1$ with probability greater than $\frac{2}{3}$, and if $\cs$ is random then $\ca$ returns $1$ with probability less than $\frac{1}{3}$.

Consider a DNF formula $\psi$ with $k \cdot \frac{n(n-1)}{2} + k + n \cdot \frac{k(k-1)}{2}$ terms such that $\psi(\bz)=1$ iff at least one of the first $k$ size-$n$ slices in $\bz$ contains $0$ more than once or less than once, or that from the first $k$ slices in $\bz$ there are two slices that encode the same index. Namely, $\psi$ return $1$ iff $\bz$ is not an extended encoding of a hyperedge.
The construction of such a formula $\psi$ is straightforward.
By Lemma~\ref{lemma:from P to DNF}, there is a DNF formula $\psi_\bx$ over $\{0,1\}^{kn}$ with at most $2^k$ terms, such that for every hyperedge $S$ we have $P_\bx(\bz^S)=\psi_\bx(\bz^S)$.
Let $\psi' = \psi \vee \psi_\bx$. Note that $\psi'$ consists of $k \cdot \frac{n(n-1)}{2} + k + n \cdot \frac{k(k-1)}{2} + 2^k$ terms, which is at most $n^3$ (for a sufficiently large $n$).
Also, note that the inputs to $\psi'$ are in $\{0,1\}^{n^{1+3/\epsilon}}$, but it uses only the first $kn$ components of the input.

If $\cs$ is pseudorandom then the examples $(\bz'_i,y'_i)$ returned by the oracle satisfy $y'_i=\psi'(\bz'_i)$. Indeed, if $\bz'_i$ is an extended encoding of a hyperedge $S_i$ then $\psi(\bz'_i)=0$ and $y'_i=P_\bx(\bz^{S_i})=\psi_\bx(\bz^{S_i})$, and otherwise $y'_i=\psi(\bz'_i)=1$.
Hence, if $\cs$ is pseudorandom then with probability at least $\frac{3}{4}$ the algorithm $\cl$ returns a hypothesis $h$ such that $\E_{\bz \sim \cd}\onefunc(h(\bz) \neq \psi'(\bz)) \leq \frac{1}{n}$.
Therefore, $\E_{\cs'_I}\ell_I(h) \leq \frac{1}{n}$.

If $\cs$ is random, then for every $i$ such that $\bz'_i$ is an extended encoding of a hyperedge $S_i$, we have $y'_i=1$ w.p. $\frac{1}{2}$ and $y'_i=0$ otherwise, and $y'_i$ is independent of $S_i$. Hence, for every $h$ and $i \in I$ we have
\begin{align*}
\Pr\left[h(\bz'_i) \neq y'_i\right]
&\geq \Pr\left[h(\bz'_i) \neq y'_i \;|\; \bz'_i \text{ represents a hyperedge}\right] \cdot \Pr\left[\bz'_i \text{ represents a hyperedge}\right]
\\
&\stackrel{(Eq.~\ref{eq:dnf-random encoding is hyperedge})}{\geq} \frac{1}{2} \cdot \frac{1}{\log(n)}
= \frac{1}{2\log(n)}~.
\end{align*}
Thus, $\E_{\cs'_I}\ell_I(h) \geq \frac{1}{2\log(n)}$.

By Lemma~\ref{lemma:hoefding1} (with $c=0$), we have for a sufficiently large $n$ that
\[
\Pr_{\cs'_I}\left[\left|\ell_{I}(h) -  \E_{\cs'_I}\ell_{I}(h)\right| \geq \frac{1}{n}\right]
< \frac{1}{20}~.
\]
Therefore, if $\cs$ is pseudorandom, then for a sufficiently large $n$, we have with probability at least $1-\left(\frac{1}{4} + \frac{1}{20}\right) = \frac{7}{10} > \frac{2}{3}$ that $\E_{\cs'_I}\ell_{I}(h) \leq \frac{1}{n}$ and $\left|\ell_{I}(h) -  \E_{\cs'_I}\ell_{I}(h)\right| < \frac{1}{n}$, and hence $\ell_{I}(h) \leq \frac{2}{n}$. Thus, the algorithm $\ca$ returns $1$ with probability greater than $\frac{2}{3}$.
If $\cs$ is random then $\E_{\cs'}\ell_{I}(h) \geq \frac{1}{2\log(n)}$ and for a sufficiently large $n$ we have with probability at least $\frac{19}{20}$ that $\left|\ell_{I}(h) - \E_{\cs'_I}\ell_{I}(h)\right| < \frac{1}{n}$. Hence, with probability greater than $\frac{2}{3}$ we have $\ell_{I}(h) \geq \frac{1}{2\log(n)} - \frac{1}{n} > \frac{2}{n}$ and the algorithm $\ca$ returns $0$.

Hence, it is hard to learn DNF formulas with $n^3$ terms where the input distribution is $\cd$.
Thus, for $\tn=n^{1+3/\epsilon}$, we have that it is hard to learn DNF formulas with $\tn^\epsilon = n^{(1+3/\epsilon) \cdot \epsilon} = n^{\epsilon + 3} \geq n^3$ terms on a distribution over $\{0,1\}^\tn$, where each component is drawn i.i.d. from a Bernoulli distribution.
\end{proof}

\subsection{Proof of Theorem~\ref{thm:circuit}}
\label{app:proof circuit}

Let $\cd$ be the uniform distribution on $\{0,1\}^{n^{1+2/\epsilon}}$.
Assume that there is an efficient algorithm $\cl$ that learns depth-$3$ Boolean circuits of size $n^2$ on the distribution $\cd$.
Let $m(n)$ be a polynomial such that $\cl$ uses a sample of size at most $m(n)$ and returns with probability at least $\frac{3}{4}$ a hypothesis $h$ with error at most $\frac{1}{2}-\gamma$.
Let $s>1$ be a constant such that $n^s \geq m(n)+n$ for every sufficiently large $n$. By Assumption~\ref{ass:localPRG}, there exists a constant $k$ and a predicate $P:\{0,1\}^k \rightarrow \{0,1\}$, such that $\cf_{P,n,n^s}$ is $\frac{1}{3}$-PRG.
We will show an algorithm $\ca$ with distinguishing advantage greater than $\frac{1}{3}$ and thus reach a contradiction.

For a hyperedge $S$ we denote by $\bz^S \in \{0,1\}^{kn}$ the encoding of $S$ that is defined in the proof of Theorem~\ref{thm:DNF}. The {\em compressed encoding} of $S$, denoted by $\tilde{\bz}^S \in \{0,1\}^{k \log(n)}$,  is a concatenation of $k$ size-$\log(n)$ slices, such that the $i$-th slice is a binary representation of the $i$-th member in $S$. We sometimes denote the $i$-th slice of $\tilde{\bz} \in \{0,1\}^{k \log(n)}$ by $\tilde{z}_{i,1},\ldots,\tilde{z}_{i,\log(n)}$.
For $\bx \in \{0,1\}^n$, let $P_\bx:\{0,1\}^{kn} \rightarrow \{0,1\}$ and $\tilde{P}_\bx:\{0,1\}^{k\log(n)} \rightarrow \{0,1\}$ be such that for every hyperedge $S$ we have $P_\bx(\bz^S) = \tilde{P}_\bx(\tilde{\bz}^S) = P(\bx_{S})$.
We say that $\tilde{\bz} \in \{0,1\}^{n^{1+2/\epsilon}}$ is an {\em extended compressed encoding} of a hyperedge $S$, if $(\tilde{z}_1,\ldots,\tilde{z}_{k\log(n)})=\tilde{\bz}^S$, namely, $\tilde{\bz}$ starts with the compressed encoding $\tilde{\bz}^S$.

If $\tilde{\bz}$ is drawn from the uniform distribution on $\{0,1\}^{n^{1+2/\epsilon}}$, then, for a sufficiently large $n$, the probability that it is an extended compressed encoding of a hyperedge, namely, that each two of the first $k$ size-$\log(n)$ slices encode different indices, is
\begin{equation}
\label{eq:circuit-random encoding is hyperedge}
\frac{n \cdot (n-1) \cdot \ldots \cdot(n-k+1)}{n^k}
\geq \left(\frac{n-k}{n}\right)^k
= \left(1-\frac{k}{n}\right)^k
\geq 1 - \frac{\gamma}{2}~.
\end{equation}

Given a sequence $(S_1,y_1),\ldots,(S_{n^s},y_{n^s})$, where $S_1,\ldots,S_{n^s}$ are i.i.d. random hyperedges, the algorithm $\ca$ needs to distinguish whether $\by = (y_1,\ldots,y_{n^s})$ is random or that $\by = (P(\bx_{S_1}),\ldots,P(\bx_{S_{n^s}})) = (\tilde{P}_\bx(\tilde{\bz}^{S_1}),\ldots,\tilde{P}_\bx(\tilde{\bz}^{S_{n^s}}))$ for a random $\bx \in \{0,1\}^n$.
Let $\cs = (\tilde{\bz}^{S_1},y_1),\ldots,(\tilde{\bz}^{S_{n^s}},y_{n^s})$.

We use the efficient algorithm $\cl$ in order to obtain distinguishing advantage greater than $\frac{1}{3}$ as follows.
The algorithm $\ca$ runs $\cl$ with the following examples oracle.
In the $i$-th call to the oracle, it chooses $\tilde{\bz}_i \in \{0,1\}^{n^{1+3/\epsilon}}$ according to $\cd$.
If $\tilde{\bz}_i$ is not an extended compressed encoding of a hyperedge (with probability at most $\frac{\gamma}{2}$, by Eq.~\ref{eq:circuit-random encoding is hyperedge}), then the oracle returns $(\bz'_i,y'_i)$ where $\bz'_i=\tilde{\bz}_i$ and $y'_i=1$. Otherwise, the oracle obtained a vector $\bz'_i$ by replacing the first $k\log(n)$ components in $\tilde{\bz}_i$ with $\tilde{\bz}^{S_i}$, and returns $(\bz'_i,y'_i)$ where $y'_i=y_i$. Note that the vector $\bz'_i$ returned by the oracle has the distribution $\cd$, since replacing a random hyperedge with another random hyperedge does not change the distribution. Let $h$ be the hypothesis returned by $\cl$.
Recall that $\cl$ uses at most $m(n)$ examples, and hence $\cs$ contains at least $n$ examples that $\cl$ cannot view. We denote the indices of these examples by $I = \{m(n)+1,\ldots,m(n)+n\}$, and the examples by $\cs_I = \{(\tilde{\bz}^{S_i},y_i)\}_{i \in I}$. By $n$ additional calls to the oracle, the algorithm $\ca$ obtains the examples $\cs'_I = \{(\bz'_i,y'_i)\}_{i \in I}$ that correspond to $\cs_I$.
Let $\ell_{I}(h)=\frac{1}{|I|}\sum_{i \in I}\onefunc(h(\bz'_i) \neq y'_i)$.
Now, if $\ell_I(h) \leq \frac{1}{2} - \frac{\gamma}{2}$, then $\ca$ returns $1$, and otherwise it returns $0$.
Clearly, the algorithm $\ca$ runs in polynomial time.
We now show that if $\cs$ is pseudorandom then $\ca$ returns $1$ with probability greater than $\frac{2}{3}$, and if $\cs$ is random then $\ca$ returns $1$ with probability less than $\frac{1}{3}$.

Consider the encoding $\bz^S$ and the compressed encoding $\tilde{\bz}^S$ of a hyperedge $S$. Note that for every $i \in [k]$ and $j \in [n]$, we have $z^S_{i,j}=0$ iff $(\tilde{z}^S_{i,1},\ldots,\tilde{z}^S_{i,\log(n)})$ is the binary representation of $j$. Hence, we can express $\neg z^S_{i,j}$ by a conjunction with the variables $\tilde{\bz}^S$, and express $z^S_{i,j}$ by a disjunction with the variables $\tilde{\bz}^S$.
By Lemma~\ref{lemma:from P to DNF}, there is a DNF formula $\psi_\bx$ over $\{0,1\}^{kn}$ with at most $2^k$ terms, such that for every hyperedge $S$ we have $P_\bx(\bz^S)=\psi_\bx(\bz^S)$.
Let $C_\bx$ be a depth-$3$ Boolean circuit such that for every hyperedge $S$ we have $C_\bx(\tilde{\bz}^S)=\psi_\bx(\bz^S)$. The circuit $C_\bx$ is obtained from $\psi_\bx$ by replacing every literal $z_{i,j}$ with the appropriate disjunction. Hence, we have $C_\bx(\tilde{\bz}^S)=\psi_\bx(\bz^S)=P_\bx(\bz^S)=\tilde{P}_\bx(\tilde{\bz}^S)$.
Note that $C_\bx$ has $1+2^k+nk \leq \frac{n^2}{2}$ gates (for a sufficiently large $n$).

Consider a DNF formula $\psi$ over $\{0,1\}^{n^{1+2/\epsilon}}$ with $n \cdot \frac{k(k-1)}{2} \leq \frac{n^2}{2}$ terms (for a sufficiently large $n$) such that $\psi(\tilde{\bz})=1$ iff $\tilde{\bz}$ is not an extended compressed encoding of a hyperedge, namely, from the first $k$ size-$\log(n)$ slices in $\tilde{\bz}$ there are two slices that encode the same index.
The construction of such a formula $\psi$ is straightforward.
Let $C'$ be a depth-$3$ Boolean circuit such that $C' = C_\bx \vee \psi$.
Note that the inputs to $C'$ are in $\{0,1\}^{n^{1+2/\epsilon}}$, but it uses only the first $k\log(n)$ components of the input.
The circuit $C'$ has at most $n^2$ gates.

If $\cs$ is pseudorandom then the examples $(\bz'_i,y'_i)$ returned by the oracle satisfy $y'_i=C'(\bz'_i)$. Indeed, if $\bz'_i$ is an extended compressed encoding of a hyperedge $S_i$ then $\psi(\bz'_i)=0$ and $y'_i=\tilde{P}_\bx(\tilde{\bz}^{S_i})=C_\bx(\tilde{\bz}^{S_i})$, and otherwise $y'_i=\psi(\bz'_i)=1$.
Hence, if $\cs$ is pseudorandom then with probability at least $\frac{3}{4}$ the algorithm $\cl$ returns a hypothesis $h$ such that $\E_{\tilde{\bz} \sim \cd}\onefunc(h(\tilde{\bz}) \neq C'(\tilde{\bz})) \leq \frac{1}{2}-\gamma$.
Therefore, $\E_{\cs'_I}\ell_I(h) \leq \frac{1}{2}-\gamma$.

If $\cs$ is random, then for every $i$ such that $\bz'_i$ is an extended compressed encoding of a hyperedge $S_i$, we have $y'_i=1$ w.p. $\frac{1}{2}$ and $y'_i=0$ otherwise, and $y'_i$ is independent of $S_i$. Hence, for every $h$ and $i \in I$ we have
\begin{align*}
\Pr\left(h(\bz'_i) \neq y'_i\right)
&\geq \Pr\left[h(\bz'_i) \neq y'_i \;|\; \bz'_i \text{ represents a hyperedge}\right] \cdot \Pr\left[\bz'_i \text{ represents a hyperedge}\right]
\\
&\stackrel{(Eq.~\ref{eq:circuit-random encoding is hyperedge})}{\geq} \frac{1}{2} \cdot \left(1 - \frac{\gamma}{2}\right)
= \frac{1}{2} - \frac{\gamma}{4}~.
\end{align*}
Thus, $\E_{\cs'_I}\ell_I(h) \geq \frac{1}{2} - \frac{\gamma}{4}$.

By the Hoefding bound, for a sufficiently large $n$ we have
\[
\Pr_{\cs'_i}\left[\left|\ell_{I}(h) -  \E_{\cs'_i}\ell_{I}(h)\right| \geq \frac{\gamma}{4} \right] \leq \frac{1}{20}~.
\]
Therefore, if $\cs$ is pseudorandom then for a sufficiently large $n$ we have with probability at least $1-\left(\frac{1}{4} + \frac{1}{20}\right) = \frac{7}{10} > \frac{2}{3}$ that $\E_{\cs'_I}\ell_{I}(h) \leq \frac{1}{2}-\gamma$ and $\left|\ell_{I}(h) -  \E_{\cs'_I}\ell_{I}(h)\right| < \frac{\gamma}{4}$, and hence $\ell_{I}(h) < \frac{1}{2} - \frac{3}{4}\gamma < \frac{1}{2} - \frac{\gamma}{2}$. Thus, the algorithm $\ca$ returns $1$ with probability greater than $\frac{2}{3}$.
If $\cs$ is random then $\E_{\cs'_I}\ell_{I}(h) \geq \frac{1}{2} - \frac{\gamma}{4}$, and for a sufficiently large $n$ we have with probability at least $\frac{19}{20}$ that $\left|\ell_{I}(h) - \E_{\cs'_I}\ell_{I}(h)\right| < \frac{\gamma}{4}$. Hence, with probability greater than $\frac{2}{3}$ we have $\ell_{I}(h) > \frac{1}{2} - \frac{\gamma}{2}$ and the algorithm $\ca$ returns $0$.

Hence, it is hard to weakly-learn depth-$3$ Boolean circuits of size at most $n^2$ where the input distribution is $\cd$.
Thus, for $\tn=n^{1+2/\epsilon}$, we have that it is hard to weakly-learn depth-$3$ Boolean circuits of size $\tn^\epsilon = n^{(1+2/\epsilon) \cdot \epsilon} = n^{\epsilon + 2} \geq n^2$, on the uniform distribution over $\{0,1\}^\tn$.

\subsection{Proof of Theorem~\ref{thm:intersections}}
\label{app:proof intersections}

By Assumption~\ref{ass:xormaj}, there is a constant $\alpha>0$ such that for every constant $s>1$ there is a constant $l$ such that for the predicate $P=\xormaj_{\ceil{\alpha s},l}$ the collection $\cf_{P,n,n^s}$ is $\frac{1}{3}$-PRG.
Let $\beta = \frac{1}{2.1\alpha}$, and let $k>\alpha$ be an integer constant.
By our assumption, for $s=\frac{k}{\alpha}$, there is a constant $l$ such that for the predicate $P=\xormaj_{\alpha s,l}=\xormaj_{k,l}$ the collection $\cf_{P,n,n^s}=\cf_{P,n,n^{k/\alpha}}=\cf_{P,n,n^{2.1 \beta k}}$ is $\frac{1}{3}$-PRG.
Let $\tn = \frac{(2n)(2n-1)}{2}+2n+1$.
Assume that there is an efficient algorithm $\cl$ that learns intersections of $k$ halfspaces over $\{0,1\}^{\tn}$. Assume that $\cl$ uses a sample of size $m(\tn) = \tn^{\beta k}$ and returns with probability at least $\frac{3}{4}$ a hypothesis with error at most $\frac{1}{10}$.
We will show an algorithm $\ca$ with distinguishing advantage greater than $\frac{1}{3}$ and thus reach a contradiction. It implies that an efficient algorithm that learns intersections of $k$ halfspaces over $\{0,1\}^{\tn}$ must use a sample of size greater than $\tn^{\beta k}$, and therefore runs in time $\Omega(\tn^{\beta k})$.
Note that we assume that $k>\alpha$. For $k \leq \alpha$ the claim holds trivially, since learning intersections of $k$ halfspaces on $\{0,1\}^{\tn}$ clearly requires time $\Omega(\tn)$, and we have $\tn^{\beta k} \leq \tn^{\beta \alpha} = \tn^{1/2.1} \leq \tn$.

For $\bz \in \{0,1\}^{2n}$, we denote by $\tilde{\bz} \in \{0,1\}^{\tn}$ the vector of all monomials over $\bz$ of degree at most $2$. We call $\tilde{\bz}$ the {\em monomials encoding} of $\bz$.
We encode a hyperedge $S = (i_1,\ldots,i_{k+l})$ by $\bz^S \in \{0,1\}^{2n}$, where $\bz^S$ is the concatenation of $2$ vectors in $\{0,1\}^n$, such that the first vector has $1$-bits in the indices $i_1,\ldots,i_k$ and $0$ elsewhere, and the second vector has $1$-bits in the indices $i_{k+1},\ldots,i_{k+l}$ and $0$ elsewhere. We denote by $\tilde{\bz}^S$ the monomials encoding of $\bz^S$.
For $\bx \in \{0,1\}^n$, let $P_\bx:\{0,1\}^{\tn} \rightarrow \{0,1\}$ be a function such that for every hyperedge $S$ we have $P_\bx(\tilde{\bz}^S) = P(\bx_{S})$.

\begin{lemma}
\label{lemma:from P to intersection}
For every $\bx \in \{0,1\}^n$, there is a function $g:\{0,1\}^{\tn} \rightarrow \{0,1\}$ that can be expressed by an intersection of $k$ halfspaces, such that for every hyperedge $S$  we have $g(\tilde{\bz}^S)=P_\bx(\tilde{\bz}^S)$.
\end{lemma}
\begin{proof}
For $\bz \in \{0,1\}^{2n}$, we denote $\bz^1 = (z_1,\ldots,z_n)$ and $\bz^2 = (z_{n+1},\ldots,z_{2n})$.
For every $\bz \in \{0,1\}^{2n}$ and even $i \in [k]$, let
\[
f_i(\bz) = \left(\inner{\bx,\bz^1}-i\right)^2 \cdot l + \left(\inner{\bx,\bz^2} - \floor*{\frac{l}{2}}\right)~,
\]
and for every odd $i \in [k]$ let
\[
f_i(\bz) = \left(\inner{\bx,\bz^1}-i\right)^2 \cdot l + 1 - \left(\inner{\bx,\bz^2} - \floor*{\frac{l}{2}}\right)~.
\]
Then, let $f:\{0,1\}^{2n} \rightarrow \{0,1\}$ be such that
\[
f(\bz) = \bigwedge_{i \in [k]} \sign(f_i(\bz))~.
\]

Note that the functions $f_i$ are degree-$2$ polynomials. For every $i \in [k]$, let $g_i: \{0,1\}^{\tn} \rightarrow \reals$ be a linear function such that for every $\bz \in \{0,1\}^{2n}$ we have $g_i(\tilde{\bz})=f_i(\bz)$, where $\tilde{\bz}$ is the monomials encoding of $\bz$. Let $g:\{0,1\}^{\tn} \rightarrow \{0,1\}$ be such that
\[
g(\tilde{\bz}) =  \bigwedge_{i \in [k]} \sign(g_i(\tilde{\bz}))~.
\]
Note that the function $g$ is an intersection of $k$ halfspaces, and that for every $\bz \in \{0,1\}^{2n}$ we have $f(\bz) = g(\tilde{\bz})$.

Assume that $\bz = \bz^S$ for a hyperedge $S=(i_1,\ldots,i_{k+l})$, and let $S^1=(i_1,\ldots,i_k)$ and $S^2=(i_{k+1},\ldots,i_{k+l})$. Note that
$\sign\left(\inner{\bx,\bz^2}- \floor{\frac{l}{2}}\right) = \maj_l(\bx_{S^2})$.
We have:
\begin{itemize}
\item If $i$ is even and $\inner{\bx,\bz^1}=i$, then $\sign(f_i(\bz)) = \sign\left(\inner{\bx,\bz^2} - \floor{\frac{l}{2}}\right) = \maj_l(\bx_{S^2})$.
\item If $i$ is odd and $\inner{\bx,\bz^1}=i$, then $\sign(f_i(\bz)) = \sign\left(1-\left(\inner{\bx,\bz^2} - \floor{\frac{l}{2}}\right)\right) = 1 - \maj_l(\bx_{S^2})$.
\item If $\inner{\bx,\bz^1} \neq i$ then $\left(\inner{\bx,\bz^1}-i\right)^2 \cdot l \geq l$. Since we also have $-\floor{\frac{l}{2}} \leq \inner{\bx,\bz^2} - \floor{\frac{l}{2}} \leq \ceil{\frac{l}{2}}$, then $\sign(f_i(\bz))=1$.
\end{itemize}

Since for every $i$ such that $\inner{\bx,\bz^1} \neq i$ we have $\sign(f_i(\bz))=1$, then $f(\bz) = \sign(f_{\inner{\bx,\bz^1}}(\bz))$.
Hence, if $\inner{\bx,\bz^1}$ is even then $f(\bz)=\maj_l(\bx_{S^2})$, and otherwise $f(\bz)=1-\maj_l(\bx_{S^2})$.
Note that $\inner{\bx,\bz^1}$ is the Hamming weight of $\bx_{S^1}$. Therefore, we have
\[
f(\bz)
= [\neg \xor_k(\bx_{S^1}) \wedge \maj_l(\bx_{S^2})] \vee [\xor_k(\bx_{S^1}) \wedge \neg \maj_l(\bx_{S^2})]
= \xormaj_{k,l}(\bx_S)
= P(\bx_S)~.
\]

For $\tilde{\bz} = \tilde{\bz}^S$, we have $g(\tilde{\bz})=f(\bz)=P(\bx_S)=P_\bx(\tilde{\bz})$. Since $g$ is an intersection of $k$ halfspaces, the lemma follows.
\end{proof}

Given a sequence $(S_1,y_1),\ldots,(S_{n^{2.1\beta k}},y_{n^{2.1\beta k}})$, where $S_1,\ldots,S_{n^{2.1\beta k}}$ are i.i.d. random hyperedges, the algorithm $\ca$ needs to distinguish whether $\by = (y_1,\ldots,y_{n^{2.1\beta k}})$ is random, or that $\by = (P(\bx_{S_1}),\ldots,P(\bx_{S_{n^{2.1\beta k}}})) = (P_\bx(\tilde{\bz}^{S_1}),\ldots,P_\bx(\tilde{\bz}^{S_{n^{2.1\beta k}}}))$ for a random $\bx \in \{0,1\}^n$.
Let $\cs = ((\tilde{\bz}^{S_1},y_1),\ldots,(\tilde{\bz}^{S_{n^{2.1\beta k}}},y_{n^{2.1\beta k}}))$.
Let $\cd$ be a distribution on $\{0,1\}^{\tn}$ such that $\tilde{\bz} \sim \cd$ is the monomials encoding of a random hyperedge. Note that each $\tilde{\bz}^{S_i}$ from $\cs$ is drawn i.i.d. from $\cd$.

We use the efficient algorithm $\cl$ in order to obtain distinguishing advantage greater than $\frac{1}{3}$ as follows.
The algorithm $\ca$ learns a function $h:\{0,1\}^{\tn} \rightarrow \{0,1\}$ by running $\cl$ with an examples oracle that in each call returns the next example from $\cs$. Recall that $\cl$ uses at most $m(\tn)=\tn^{\beta k}$ examples, and hence $\cs$ contains at least
\[
n^{2.1\beta k}-\tn^{\beta k}
\geq n^{2.1 \beta k}-(2n)^{2\beta k}
\geq n^{2.1 \beta k}-(n^{1.01})^{2\beta k}
= n^{2.1 \beta k}-n^{2.02\beta k}
= n^{2.02 \beta k}(n^{0.08 \beta k} - 1)
\geq \ln(n)
\]
examples that $\cl$ cannot view (for a sufficiently large $n$).
We denote the indices of these examples by $I = \{m(\tn)+1,\ldots,m(\tn)+\ln(n)\}$, and the examples by $\cs_I = \{(\tilde{\bz}^{S_i},y_i)\}_{i \in I}$.
Let $\ell_{I}(h)=\frac{1}{|I|}\sum_{i \in I}\onefunc(h(\tilde{\bz}^{S_i}) \neq y_i)$.
Now, if $\ell_I(h) \leq \frac{2}{10}$, then $\ca$ returns $1$, and otherwise it returns $0$.
Clearly, the algorithm $\ca$ runs in polynomial time.
We now show that if $\cs$ is pseudorandom then $\ca$ returns $1$ with probability greater than $\frac{2}{3}$, and if $\cs$ is random then $\ca$ returns $1$ with probability less than $\frac{1}{3}$.

If $\cs$ is pseudorandom, then by Lemma~\ref{lemma:from P to intersection}, it can be realized by an intersection of $k$ halfspaces. Hence, with probability at least $\frac{3}{4}$ the algorithm $\cl$ returns a function $h$, such that $\E_{\tilde{\bz} \sim \cd}\onefunc(h(\tilde{\bz}) \neq P_\bx(\tilde{\bz})) \leq \frac{1}{10}$. Therefore, $\E_{\cs_I}\ell_{I}(h) \leq \frac{1}{10}$.
If $\cs$ is random then for every function $h:\{0,1\}^{\tn} \rightarrow \{0,1\}$ the events $\{h(\tilde{\bz_i}) = y_i\}_{i \in I}$ are independent from one another, and each has probability $\frac{1}{2}$. Hence, $\E_{\cs_I}\ell_{I}(h) = \frac{1}{2}$.

By the Hoefding bound, for a sufficiently large $n$ we have
\[
\Pr_{\cs_I}\left[\left|\ell_{I}(h) -  \E_{\cs_I}\ell_{I}(h)\right| \geq \frac{1}{10} \right] \leq \frac{1}{20}~.
\]
Therefore, if $\cs$ is pseudorandom then for a sufficiently large $n$ we have with probability at least $1-\left(\frac{1}{4} + \frac{1}{20}\right) = \frac{7}{10} > \frac{2}{3}$ that $\E_{\cs_I}\ell_{I}(h) \leq \frac{1}{10}$ and $\left|\ell_{I}(h) -  \E_{\cs_I}\ell_{I}(h)\right| < \frac{1}{10}$, and hence $\ell_{I}(h) \leq \frac{2}{10}$. Thus, the algorithm $\ca$ returns $1$ with probability greater than $\frac{2}{3}$.
If $\cs$ is random then $\E_{\cs}\ell_{I}(h) = \frac{1}{2}$ and for a sufficiently large $n$ we have with probability at least $\frac{19}{20}$ that $\left|\ell_{I}(h) - \E_{\cs_I}\ell_{I}(h)\right| < \frac{1}{10}$. Hence, with probability greater than $\frac{2}{3}$ we have $\ell_{I}(h) > \frac{2}{10}$ and the algorithm $\ca$ returns $0$.

\subsection{Proof of Theorem~\ref{thm:nn discrete}}
\label{app:proof nn discrete}

\begin{lemma}
\label{lemma:from square to 0-1 part 1}
Let $\cd$ be a distribution on $\{0,1\}^n$ and let $\epsilon>0$. Let $f:\{0,1\}^n \rightarrow \{0,1\}$ and $h:\{0,1\}^n \rightarrow \reals$ be functions such that $\E_{\bx \sim \cd}(f(\bx)-h(\bx))^2 \leq \frac{\epsilon}{4}$. Let $h':\{0,1\}^n \rightarrow \{0,1\}$ be such that for every $\bx \in \{0,1\}^n$ we have $h'(\bx) = \sign\left(h(\bx)-\frac{1}{2}\right)$. Then $\E_{\bx \sim \cd}\onefunc(h'(\bx) \neq f(\bx)) \leq \epsilon$.
\end{lemma}
\begin{proof}
For every $\bx$ such that $h'(\bx) \neq f(\bx)$, we have $(h(\bx)-f(\bx))^2 \geq \frac{1}{4}$. Hence,
\[
\E_{\bx \sim \cd}\onefunc(h'(\bx) \neq f(\bx))
\leq \E_{\bx \sim \cd}4(h(\bx)-f(\bx))^2
\leq 4 \cdot \frac{\epsilon}{4}
= \epsilon~.
\]
\end{proof}

\begin{lemma}
\label{lemma:from square to 0-1 part 2}
Let $\ch' \subseteq \{0,1\}^{(\{0,1\}^n)}$ be a hypothesis class, and let $\ch \subseteq \reals^{(\{0,1\}^n)}$ be a hypothesis class such that $\ch' \subseteq \ch$.
If it is hard to learn $\ch'$ with respect to the 0-1 loss, then it is hard to learn $\ch$ with respect to the square loss.
Moreover, for every distribution $\cd$ on $\{0,1\}^n$, if it is hard to learn $\ch'$ on $\cd$ with respect to the 0-1 loss, then it is hard to learn $\ch$ on $\cd$ with respect to the square loss.
\end{lemma}
\begin{proof}
Let $\epsilon,\delta \in (0,1)$.
Assume that there is an efficient algorithm $\cl$ that for every $f \in \ch$ and distribution $\cd$ on $\{0,1\}^n$, given access to examples $(\bx,f(\bx))$ where $\bx \sim \cd$, finds a hypothesis $h:\{0,1\}^n \rightarrow \reals$ such that with probability at least $1-\delta$ we have $\E_{\bx \sim \cd}(h(\bx)-f(\bx))^2 \leq \frac{\epsilon}{4}$.
Consider a learning algorithm $\cl'$, that given access to examples $(\bx,f'(\bx))$ where $\bx \sim \cd$ and $f' \in \ch'$, runs $\cl$, and returns a hypothesis $h':\{0,1\}^n \rightarrow \{0,1\}$ such that for every $\bx \in \reals^n$ we have $h'(\bx) = \sign\left(h(\bx)-\frac{1}{2}\right)$.
By Lemma~\ref{lemma:from square to 0-1 part 1}, we have $\E_{\bx \sim \cd}\onefunc(h'(\bx) \neq f(\bx)) \leq \epsilon$.
Therefore, $\ch'$ can be learned efficiently with respect to the 0-1 loss.
The same argument holds also for the case of distribution-specific learning.
\end{proof}

\subsubsection{Proof of (1)}

Note that in order to express a DNF formula with a depth-$2$ neural network, the network should have an activation function in the output neuron. Since, this is not allowed here, then the claim does not follow immediately from Lemma~\ref{lemma:DNF distribution free} and Lemma~\ref{lemma:from square to 0-1 part 2}.
Nevertheless, the proof follows similar ideas to the proof of Lemma~\ref{lemma:DNF distribution free}, with a few modifications as detailed below.

In the proof of Lemma~\ref{lemma:DNF distribution free}, we consider a sequence $\cs = ((\bz^{S_1},y_1),\ldots,(\bz^{S_{n^s}},y_{n^s}))$, and show that if $\cs$ is pseudorandom, namely, $y_i=P_\bx(\bz^{S_i})$ for all $i$, then for every hyperedge $S$ we have $P_\bx(\bz^S)=\psi_\bx(\bz^S)$, where $\psi_\bx$ is a DNF formula with at most $2^k$ terms. Then, we use the assumption that there is an efficient algorithm for learning DNF formulas with $\omega(1)$ terms, in order to obtain a hypothesis $h$ such that $\E_{\bz \sim \cd} \onefunc(h(\bz) \neq P_\bx(\bz)) \leq \frac{1}{10}$, and we use $h$ in order to obtain distinguishing advantage greater than $\frac{1}{3}$ and reach a contradiction.
Here, we will show that for every hyperedge $S$ we also have $P_\bx(\bz^S)=N_\bx(\bz^S)$, where $N_\bx$ is a depth-$2$ neural network with $2^k$ hidden neurons and no activation in the output neuron. Then, if we assume that there is an efficient algorithm for learning depth-$2$ neural networks with $\omega(1)$ hidden neurons and no activation in the output neuron, with respect to the square loss, then we can obtain a hypothesis $h$ with small error with respect to the square loss. By Lemma~\ref{lemma:from square to 0-1 part 1}, we can obtain a hypothesis $h'$ with small error with respect to the 0-1 loss. The arguments from the proof of Lemma~\ref{lemma:DNF distribution free} then imply that we can use $h'$ in order to obtain distinguishing advantage greater than $\frac{1}{3}$ and reach a contradiction.

We now construct the neural network $N_\bx$ such that for every hyperedge $S$ we have $P_\bx(\bz^S)=\psi_\bx(\bz^S)=N_\bx(\bz^S)$.
Note that $N_\bx$ should simulate $\psi_\bx$ only for inputs that encode hyperedges, and not for all $\bz \in \{0,1\}^{kn}$.
Each term $C_j$ in $\psi_\bx$ is a conjunction of positive literals. Let $I_j \subseteq [kn]$ be the indices of these literals.
Note that $C_j(\bz^{S})$ can be expressed by a single ReLU neuron that computes $\left[\left(\sum_{l \in I_j}z^{S}_l\right) - \left(|I_j|-1\right)\right]_+$.
Thus, our neural network $N_\bx$ includes a hidden neuron for every term in $\psi_\bx$.
By the construction in Lemma~\ref{lemma:from P to DNF}, each conjunction $C_j(\bz^{S})$ checks whether $\bx_{S}$ is the $j$-th satisfying assignment of the predicate $P$. Hence, it is not possible that more than one term in $\psi_\bx(\bz^{S})$ is satisfied. Therefore, the network $N_\bx$ computes $\psi_\bx(\bz^{S})$ by summing the outputs of the hidden neurons, and since this sum is in $\{0,1\}$ then an activation function is not required in the output neuron.

\subsubsection{Proof of (2) and (3)}

Implementing a depth-$d$ Boolean circuit with a depth-$d$ neural network is straightforward. Hence, the claims follow immediately from Theorems~\ref{thm:DNF} and~\ref{thm:circuit}, and from Lemma~\ref{lemma:from square to 0-1 part 2}.

\subsection{Proof of Theorem~\ref{thm:nn fixed k}}
\label{app:proof nn fixed k}

Note that in order to express an intersection of halfspaces with a depth-$2$ neural network, the network should have an activation function in the output neuron. Since, this is not allowed here, then the claim does not follow immediately from Theorem~\ref{thm:intersections} and Lemma~\ref{lemma:from square to 0-1 part 1}.
Nevertheless, the proof follows similar ideas to the proof of Theorem~\ref{thm:intersections}, with a few modifications as detailed below.

In the proof of Theorem~\ref{thm:intersections}, we consider a sequence $\cs = ((\tilde{\bz}^{S_1},y_1),\ldots,(\tilde{\bz}^{S_{n^{2.1\beta k}}},y_{n^{2.1\beta k}}))$, and show that if $\cs$ is pseudorandom, namely, $y_i=P_\bx(\tilde{\bz}^{S_i})$ for all $i$, then for every hyperedge $S$ we have $P_\bx(\tilde{\bz}^S)=g_\bx(\tilde{\bz}^S)$, where $g_\bx$ is an intersection of $k$ halfspaces. Then, we use the assumption that there is an efficient algorithm that learns an intersections of $k$ halfspaces and uses a sample of size $\tn^{\beta k}$, in order to obtain a hypothesis $h$ such that $\E_{\bz \sim \cd} \onefunc(h(\bz) \neq P_\bx(\bz)) \leq \frac{1}{10}$, and we use $h$ in order to obtain distinguishing advantage greater than $\frac{1}{3}$ and reach a contradiction.
Here, we will show that for every hyperedge $S$ we also have $P_\bx(\tilde{\bz}^S)=N_\bx(\tilde{\bz}^S)$, where $N_\bx$ is a depth-$2$ neural networks with $2k$ hidden neurons and no activation in the output neuron. Then, if we assume that there is an efficient algorithm that learns
such networks with respect to the square loss and uses a sample of size $\tn^{\beta k}$, then we can obtain a hypothesis $h$ with small error with respect to the square loss. By Lemma~\ref{lemma:from square to 0-1 part 1}, we can obtain a hypothesis $h'$ with small error with respect to the 0-1 loss. The arguments from the proof of Theorem~\ref{thm:intersections} then imply that we can use $h'$ in order to obtain distinguishing advantage greater than $\frac{1}{3}$ and reach a contradiction.

We now construct the neural network $N_\bx$ such that for every hyperedge $S$ we have $P_\bx(\tilde{\bz}^S)=g_\bx(\tilde{\bz}^S)=N_\bx(\tilde{\bz}^S)$.
Note that $N_\bx$ should simulate $g_\bx$ only for inputs that encode hyperedges, and not for all $\tilde{\bz} \in \{0,1\}^{\tn}$.
By Lemma~\ref{lemma:from P to intersection}, $g_\bx(\tilde{\bz}^S)=\bigwedge_{i \in [k]} \sign(g_i(\tilde{\bz}^S))$, and we show there that for every hyperedge $S$, there is at most one index $i$ with $\sign(g_i(\tilde{\bz}^S))=0$. Hence,
\[
g_\bx(\tilde{\bz}^S)
= \sum_{i \in [k]} \sign(g_i(\tilde{\bz}^S)) - \left(k-1\right)~.
\]
The network $N_\bx$ computes $\sign(g_i(\tilde{\bz}^S))$ for every $i \in [k]$ using a single nonlinear layer. Then, the computation of $g_\bx(\tilde{\bz}^S)$ does not require an activation function in the output neuron. Note that the output neuron does not have to include a bias term, since the additive term $-(k-1)$ can be implemented by adding a hidden neuron with fan-in $0$ and bias $k-1$, that is connected to the output neuron with weight $1$.

For every $i \in [k]$ and $\tilde{\bz} \in \{0,1\}^\tn$, the network $N_\bx$ computes $\sign(g_i(\tilde{\bz}))$ as follows.
We denote $g_i(\tilde{\bz}) = \inner{\bw_i,\tilde{\bz}}$.
Since $\inner{\bw_i,\tilde{\bz}}$ is an integer, we have
\[
\sign(g_i(\tilde{\bz})) =
\left[\inner{\bw_i,\tilde{\bz}} \right]_+ - \left[\inner{\bw_i,\tilde{\bz}} - 1 \right]_+~.
\]
Hence, computing $\sign(g_i(\tilde{\bz}^S))$ requires $2$ hidden neurons. Therefore, the network $N_\bx$ includes $2k$ hidden neurons.

\subsection{Proof of Theorem~\ref{thm:nn normal}}
\label{app:proof nn normal}

Let $\cd$ be the standard Gaussian distribution on $\reals^{n^{1+3/\epsilon}}$.
Assume that there is an efficient algorithm $\cl$ that learns depth-$3$ neural networks with $n^3$ hidden neurons on the distribution $\cd$.
Let $m(n)$ be a polynomial such that $\cl$ uses a sample of size at most $m(n)$ and returns with probability at least $\frac{3}{4}$ a hypothesis $h$ with error at most $\frac{1}{n}$.
Let $s>1$ be a constant such that $n^s \geq m(n)+n^3$ for every sufficiently large $n$. By Assumption~\ref{ass:localPRG}, there exists a constant $k$ and a predicate $P:\{0,1\}^k \rightarrow \{0,1\}$, such that $\cf_{P,n,n^s}$ is $\frac{1}{3}$-PRG.
We will show an algorithm $\ca$ with distinguishing advantage greater than $\frac{1}{3}$ and thus reach a contradiction.

For a hyperedge $S$ we denote by $\bz^S \in \{0,1\}^{kn}$ the encoding of $S$ that is defined in the proof of Theorem~\ref{thm:DNF}.
For $\bx \in \{0,1\}^n$, let $P_\bx:\{0,1\}^{kn} \rightarrow \{0,1\}$ be such that for every hyperedge $S$ we have $P_\bx(\bz^S) = P(\bx_{S})$.
We denote by $\cn(0,1)$ the standard univariate normal distribution. Let $c$ be a constant such that $\Pr_{t \sim \cn(0,1)}[t \leq c] = \frac{1}{n}$. Let $\mu$ be the density of $\cn(0,1)$, let $\mu_-(t) = n \cdot \onefunc(t \leq c) \cdot \mu(t)$, and let $\mu_+(t) = \frac{n}{n-1} \cdot \onefunc(t \geq c) \cdot \mu(t)$.
Let $\Psi:\reals^{kn} \rightarrow \{0,1\}^{kn}$ be a mapping such that for every $\bz' \in \reals^{kn}$ and $i \in [kn]$ we have $\Psi(\bz')_i = 1$ iff $z'_i \geq c$.
For $\tilde{\bz} \in \reals^{n^{1+3/\epsilon}}$ we denote $\tilde{\bz}_{[kn]}=(\tilde{z}_1,\ldots,\tilde{z}_{kn})$, namely, the first $kn$ component of $\tilde{\bz}$.

Let $N_1:\reals^{kn} \rightarrow [0,2^k]$ be a depth-$3$ neural network with at most $\frac{n^3}{3}$ hidden neurons (for a sufficiently large $n$), and no activation function in the output neuron, that satisfies the following property. Let $\bz' \in \reals^{kn}$ be such that $\Psi(\bz')=\bz^S$ for some hyperedge $S$, and assume that for every $i \in [kn]$ we have $z'_i \not \in (c,c+\frac{1}{n^2})$, then $N_1(\bz')=P_\bx(\bz^S)$. The construction of the network $N_1$ is given in Lemma~\ref{lemma:network N1}.
Let $N_2:\reals^{kn} \rightarrow \reals_+$ be a depth-$3$ neural network with at most $\frac{n^3}{3}$ hidden neurons (for a sufficiently large $n$), and no activation function in the output neuron, that satisfies the following property. Let $\bz' \in \reals^{kn}$ be such that for every $i \in [kn]$ we have $z'_i \not \in (c,c+\frac{1}{n^2})$. If $\Psi(\bz')$ is an encoding of a hyperedge then $N_2(\bz')=0$, and otherwise $N_2(\bz') \geq 2^k$. The construction of the network $N_2$ is given in Lemma~\ref{lemma:network N2}.
Let $N_3:\reals^{kn} \rightarrow \reals_+$ be a depth-$2$ neural network with at most $\frac{n^3}{3}$ hidden neurons (for a sufficiently large $n$), such that for $\bz' \in \reals^{kn}$ we have: If there exists $i \in [kn]$ such that $z'_i \in (c,c+\frac{1}{n^2})$ then $N_3(\bz') \geq 2^k$, and if for every $i \in [kn]$ we have $z'_i \not \in (c-\frac{1}{n^2},c+\frac{2}{n^2})$ then $N_3(\bz')=0$. The construction of the network $N_3$ is given in Lemma~\ref{lemma:network N3}.
Note that the network $N_1$ depends on $\bx$, and the networks $N_2,N_3$ are independent of $\bx$.
Let $N'$ be a depth-$3$ neural network such that for every $\bz' \in \reals^{kn}$ we have $N'(\bz') = [N_1(\bz') - N_2(\bz') - N_3(\bz')]_+$.
The network $N'$ has at most $n^3$ hidden neurons. We note that all weights in $N'$ are bounded by some $\poly(n)$ that is independent of $k$, namely, using weights of magnitude $n^k$ is not allowed. This is crucial since we need to show hardness of learning already where the weights of the network are bounded.
Let $\tilde{N}:\reals^{n^{1+3/\epsilon}} \rightarrow \reals$ be a depth-$3$ neural network such that $\tilde{N}(\tilde{\bz}) = N'(\tilde{\bz}_{[kn]})$.

Given a sequence $(S_1,y_1),\ldots,(S_{n^s},y_{n^s})$, where $S_1,\ldots,S_{n^s}$ are i.i.d. random hyperedges, the algorithm $\ca$ needs to distinguish whether $\by = (y_1,\ldots,y_{n^s})$ is random or that $\by = (P(\bx_{S_1}),\ldots,P(\bx_{S_{n^s}})) = (P_\bx(\bz^{S_1}),\ldots,P_\bx(\bz^{S_{n^s}}))$ for a random $\bx \in \{0,1\}^n$.
Let $\cs = ((\bz^{S_1},y_1),\ldots,(\bz^{S_{n^s}},y_{n^s}))$.

We use the efficient algorithm $\cl$ in order to obtain distinguishing advantage greater than $\frac{1}{3}$ as follows.
The algorithm $\ca$ runs $\cl$ with the following examples oracle.
In the $i$-th call, the oracle first draws $\bz \in \{0,1\}^{kn}$ such that each component is drawn i.i.d. from a Bernoulli distribution where the probability of $0$ is $\frac{1}{n}$. If $\bz$ is an encoding of a hyperedge then the oracle replaces $\bz$ with $\bz^{S_i}$. Then, the oracle chooses $\bz' \in \reals^{kn}$ such that for each component $j$, if $z_j \geq c$ then $z'_j$ is drawn from $\mu_+$, and otherwise $z'_j$ is drawn from $\mu_-$.
Let $\tilde{\bz} \in \reals^{n^{1+3/\epsilon}}$ be such that $\tilde{\bz}_{[kn]}=\bz'$, and the other $n^{1+3/\epsilon}-kn$ components of $\tilde{\bz}$ are drawn i.i.d. from $\cn(0,1)$.
Note that the vector $\tilde{\bz}$ has the distribution $\cd$, due to the definitions of the densities $\mu_+$ and $\mu_-$, and since replacing an encoding of a random hyperedge by an encoding of another random hyperedge does not change the distribution of $\bz$.
The oracle returns $(\tilde{\bz},\tilde{y})$, where the labels $\tilde{y}$ are chosen as follows:
\begin{itemize}
\item If $\Psi(\bz')$ is not an encoding of a hyperedge, then $\tilde{y}=0$.
\item If $\Psi(\bz')$ is an encoding of a hyperedge:
    \begin{itemize}
    \item If $\bz'$ does not have components in the interval $(c-\frac{1}{n^2},c+\frac{2}{n^2})$, then $\tilde{y} = y_i$.
    \item If $\bz'$ has a component in the interval $(c,c+\frac{1}{n^2})$, then $\tilde{y} = 0$.
    \item If $\bz'$ does not have components in the interval $(c,c+\frac{1}{n^2})$, but has a component in the interval $(c-\frac{1}{n^2},c+\frac{2}{n^2})$, then $\tilde{y} = [y_i - N_3(\bz')]_+$.
    \end{itemize}
\end{itemize}

Let $h$ be the hypothesis returned by $\cl$.
Recall that $\cl$ uses at most $m(n)$ examples, and hence $\cs$ contains at least $n^3$ examples that $\cl$ cannot view. We denote the indices of these examples by $I = \{m(n)+1,\ldots,m(n)+n^3\}$, and the examples by $\cs_I = \{(\bz^{S_i},y_i)\}_{i \in I}$. By $n^3$ additional calls to the oracle, the algorithm $\ca$ obtains the examples $\tilde{\cs}_I = \{(\tilde{\bz}_i,\tilde{y}_i)\}_{i \in I}$ that correspond to $\cs_I$.
Let $\ell_{I}(h)=\frac{1}{|I|}\sum_{i \in I}(h(\tilde{\bz}_i)-\tilde{y}_i)^2$.
Now, if $\ell_I(h) \leq \frac{2}{n}$, then $\ca$ returns $1$, and otherwise it returns $0$.
Clearly, the algorithm $\ca$ runs in polynomial time.
We now show that if $\cs$ is pseudorandom then $\ca$ returns $1$ with probability greater than $\frac{2}{3}$, and if $\cs$ is random then $\ca$ returns $1$ with probability less than $\frac{1}{3}$.

In Lemma~\ref{lemma:realizable by N}, we show that if $\cs$ is pseudorandom then the examples $(\tilde{\bz},\tilde{y})$ returned by the oracle are realized by $\tilde{N}$.
Hence, with probability at least $\frac{3}{4}$ the algorithm $\cl$ returns a hypothesis $h$ such that $\E_{\tilde{\bz} \sim \cd}(h(\tilde{\bz})-\tilde{N}(\tilde{\bz}))^2 \leq \frac{1}{n}$.
Therefore, $\E_{\tilde{\cs}_I}\ell_I(h) \leq \frac{1}{n}$.

Let $\tilde{\cz} \subseteq \reals^{n^{(1+3/\epsilon)}}$ be such that $\tilde{\bz} \in \tilde{\cz}$ if $\tilde{\bz}_{[kn]}$ does not have components in the interval $(c-\frac{1}{n^2},c+\frac{2}{n^2})$, and $\Psi(\tilde{\bz}_{[kn]})=\bz^{S}$ for a hyperedge $S$.
If $\cs$ is random, then for every $i$ such that $\tilde{\bz}_i \in \tilde{\cz}$, we have $\tilde{y}_i=1$ w.p. $\frac{1}{2}$ and $\tilde{y}_i=0$ otherwise. Also, by the definition of the oracle, $\tilde{y}_i$ is independent of $S_i$ and independent of the choice of the vector $\tilde{\bz}_i$ that corresponds to $\bz^{S_i}$.
Hence, for every $h$ and $i \in I$ we have
\begin{equation*}
\label{eq:nn-large error}
\Pr\left[(h(\tilde{\bz}_i)-\tilde{y}_i)^2 \geq \frac{1}{4}\right]
\geq \Pr\left[\left.(h(\tilde{\bz}_i)-\tilde{y}_i)^2 \geq \frac{1}{4} \; \right| \; \tilde{\bz}_i \in \tilde{\cz} \right] \cdot \Pr\left[\tilde{\bz}_i \in \tilde{\cz}\right]
\geq \frac{1}{2}  \cdot \Pr\left(\tilde{\bz}_i \in \tilde{\cz}\right)~.
\end{equation*}
In Lemma~\ref{lemma:prob z' good} we show that $\Pr\left[\tilde{\bz}_i \in \tilde{\cz}\right] \geq \frac{1}{2\log(n)}$.
Hence,
\[
\Pr\left[(h(\tilde{\bz}_i)-\tilde{y}_i)^2 \geq \frac{1}{4}\right]
\geq \frac{1}{4\log(n)}~.
\]
Thus,
\[
\E_{\tilde{\cs}_I}\ell_I(h) \geq \frac{1}{4} \cdot \frac{1}{4\log(n)} = \frac{1}{16\log(n)}~.
\]

By Lemma~\ref{lemma:hoefding1} (with $c=0$), we have for a sufficiently large $n$ that
\[
\Pr_{\tilde{\cs}_I}\left[\left|\ell_{I}(h) -  \E_{\tilde{\cs}_I}\ell_{I}(h)\right| \geq \frac{1}{n}\right]
< \frac{1}{20}~.
\]
Therefore, if $\cs$ is pseudorandom, then for a sufficiently large $n$, we have with probability at least $1-\left(\frac{1}{4} + \frac{1}{20}\right) = \frac{7}{10} > \frac{2}{3}$ that $\E_{\tilde{\cs}_I}\ell_{I}(h) \leq \frac{1}{n}$ and $\left|\ell_{I}(h) -  \E_{\tilde{\cs}_I}\ell_{I}(h)\right| < \frac{1}{n}$, and hence $\ell_{I}(h) \leq \frac{2}{n}$. Thus, the algorithm $\ca$ returns $1$ with probability greater than $\frac{2}{3}$.
If $\cs$ is random then $\E_{\tilde{\cs}}\ell_{I}(h) \geq \frac{1}{16\log(n)}$ and for a sufficiently large $n$ we have with probability at least $\frac{19}{20}$ that $\left|\ell_{I}(h) - \E_{\tilde{\cs}_I}\ell_{I}(h)\right| < \frac{1}{n}$. Hence, with probability greater than $\frac{2}{3}$ we have $\ell_{I}(h) > \frac{1}{16\log(n)} - \frac{1}{n} > \frac{2}{n}$ and the algorithm $\ca$ returns $0$.

Hence, it is hard to learn depth-$3$ neural networks with $n^3$ hidden neurons on the distribution $\cd$.
Thus, for $\tn=n^{1+3/\epsilon}$, we have that it is hard to learn depth-$3$ neural networks with $\tn^\epsilon = n^{(1+3/\epsilon) \cdot \epsilon} = n^{\epsilon + 3} \geq n^3$ hidden neurons on a standard Gaussian distribution over $\reals^\tn$.

\begin{lemma}
\label{lemma:network N1}
There exists a depth-$3$ neural network $N_1:\reals^{kn} \rightarrow [0,2^k]$ with at most $2kn+2^k$ hidden neurons and no activation function in the output neuron, that satisfies the following property. Let $\bz' \in \reals^{kn}$ be such that $\Psi(\bz')=\bz^S$ for some hyperedge $S$, and assume that for every $i \in [kn]$ we have $z'_i \not \in (c,c+\frac{1}{n^2})$, then $N_1(\bz')=P_\bx(\bz^S)$.
\end{lemma}
\begin{proof}
Let $N_\bx$ be the depth-$2$ neural network from the proof of Theorem~\ref{thm:nn discrete} (part $1$). The network $N_\bx$ is such that for every hyperedge $S$, we have $N_\bx(\bz^S)=P_\bx(\bz^S)$.
Also, the network $N_\bx$ is such that for every $\bz \in \reals^{kn}$, we have
\[
N_\bx(\bz)
= \sum_{1 \leq j \leq J} \left[\left(\sum_{l \in I_j}z_l\right) - \left(|I_j|-1\right)\right]_+~,
\]
where $J \leq 2^k$, and $I_j \subseteq [kn]$. 
Therefore, for every $\bz \in [0,1]^{kn}$ we have $N_\bx(\bz) \in [0,2^k]$.

Next, we construct a depth-$2$ neural network $N_\Psi:\reals^{kn} \rightarrow [0,1]^{kn}$ with a single layer of non-linearity, such that for every $\bz' \in \reals^{kn}$ with $z'_i \not \in (c,c+\frac{1}{n^2})$ for every $i \in [kn]$, we have $N_\Psi(\bz') = \Psi(\bz')$.
The network $N_\Psi$ has $2kn$ hidden neurons, and computes $N_\Psi(\bz') = (f(z'_1),\ldots,f(z'_{kn}))$, where $f:\reals \rightarrow [0,1]$ is such that
\[
f(t) = n^2 \cdot \left(\left[t - c\right]_+ - \left[t - \left(c + \frac{1}{n^2}\right)\right]_+ \right)~.
\]
Note that if $t \leq c$ then $f(t)=0$, if $t \geq c+\frac{1}{n^2}$ then $f(t)=1$, and if $c < t <  c+\frac{1}{n^2}$ then $f(t) \in (0,1)$.

The network $N_1$ is obtained by combining the networks $N_\Psi$ and $N_\bx$. Note that $N_1$ has at most $2kn+2^k$ hidden neurons, and satisfies the requirements.
\end{proof}

\begin{lemma}
\label{lemma:network N2}
There exists a depth-$3$ neural network $N_2:\reals^{kn} \rightarrow \reals_+$ with at most $2kn + k \cdot \frac{n(n-1)}{2} + k + n \cdot \frac{k(k-1)}{2}$ hidden neurons, and no activation function in the output neuron, that satisfies the following property. Let $\bz' \in \reals^{kn}$ be such that for every $i \in [kn]$ we have $z'_i \not \in (c,c+\frac{1}{n^2})$. If $\Psi(\bz')$ is an encoding of a hyperedge then $N_2(\bz')=0$, and otherwise $N_2(\bz') \geq 2^k$.
\end{lemma}
\begin{proof}
By the proof of Lemma~\ref{lemma:DNF distribution specific}, there is a DNF formula $\varphi$ over $\{0,1\}^{kn}$ with $k \cdot \frac{n(n-1)}{2} + k + n \cdot \frac{k(k-1)}{2}$ terms such that $\varphi(\bz)=1$ iff $\bz$ is not an encoding of a hyperedge. Each term in $\varphi$ can be implemented by a single ReLU neuron. By summing the outputs of these neurons and multiplying by $2^k$ we obtain a depth-$2$ neural network $N_\varphi$, such that if $\bz \in \{0,1\}^{kn}$ is an encoding of a hyperedge then $N_\varphi(\bz)=0$, and otherwise $N_\varphi(\bz) \geq 2^k$. For every $\bz \in \reals^{kn}$ we have $N_\varphi(\bz) \geq 0$.

Let $N_\Psi:\reals^{kn} \rightarrow [0,1]^{kn}$ be the depth-$2$ neural network from the proof of Lemma~\ref{lemma:network N1}, with a single layer of non-linearity of $2kn$ hidden neurons, such that for every $\bz' \in \reals^{kn}$ with $z'_i \not \in (c,c+\frac{1}{n^2})$ for every $i \in [kn]$, we have $N_\Psi(\bz') = \Psi(\bz')$.
By combining $N_\Psi$ and $N_\varphi$ we obtain a depth-$3$ network $N_2$ with $2kn + k \cdot \frac{n(n-1)}{2} + k + n \cdot \frac{k(k-1)}{2}$ hidden neurons that satisfies the requirements.
\end{proof}

\begin{lemma}
\label{lemma:network N3}
There exists a depth-$2$ neural network $N_3:\reals^{kn} \rightarrow \reals_+$ with at most $4kn$ hidden neurons, such that for $\bz' \in \reals^{kn}$ we have: If there exists $i \in [kn]$ such that $z'_i \in (c,c+\frac{1}{n^2})$ then $N_3(\bz') \geq 2^k$, and if for every $i \in [kn]$ we have $z'_i \not \in (c-\frac{1}{n^2},c+\frac{2}{n^2})$ then $N_3(\bz')=0$.
\end{lemma}
\begin{proof}
We construct a depth-$2$ network $N_3:\reals^{kn} \rightarrow [0, 2^k \cdot kn]$ with $4kn$  hidden neurons, such that $N_3(\bz')=2^k \cdot \sum_{i \in [kn]}m_i$, where
\begin{itemize}
\item If $z'_i \in (c,c+\frac{1}{n^2})$ then $m_i = 1$.
\item If $z'_i \not \in (c-\frac{1}{n^2},c+\frac{2}{n^2})$ then $m_i = 0$.
\item If $z'_i \in (c-\frac{1}{n^2},c]$ then $m_i = \left(z'_i-c+\frac{1}{n^2}\right) \cdot n^2 \in [0,1]$.
\item If $z'_i \in [c+\frac{1}{n^2},c+\frac{2}{n^2})$ then $m_i = 1 - \left(z'_i - c - \frac{1}{n^2}\right) \cdot n^2 \in [0,1]$.
\end{itemize}

The construction now follows immediately from the fact that for every $i \in [kn]$ we have
\begin{align*}
m_i = &(n^2)\left( \left[z'_i - \left(c-\frac{1}{n^2}\right)\right]_+ - \left[z'_i-c\right]_+ \right) -
\\
&(n^2)\left( \left[z'_i-\left(c+\frac{1}{n^2}\right)\right]_+ - \left[z'_i-\left(c+\frac{2}{n^2}\right)\right]_+ \right)~.
\end{align*}
\end{proof}

\begin{lemma}
\label{lemma:realizable by N}
If $\cs$ is pseudorandom then the examples $(\tilde{\bz},\tilde{y})$ returned by the oracle are realized by $\tilde{N}$.
\end{lemma}
\begin{proof}
Let $\bz' = \tilde{\bz}_{[kn]}$. Thus, $\tilde{N}(\tilde{\bz}) = N'(\bz')$. We show that $\tilde{y} = N'(\bz')$.
\begin{itemize}
\item If $\Psi(\bz')$ is not an encoding of a hyperedge, then:
    \begin{itemize}
    \item If $\bz'$ does not have components in the interval $(c,c+\frac{1}{n^2})$, then $N_1(\bz') \in [0,2^k]$, $N_2(\bz') \geq 2^k$, and $N_3(\bz') \geq 0$. Therefore, $N'(\bz')=0=\tilde{y}$.
    \item If $\bz'$ has a component in the interval $(c,c+\frac{1}{n^2})$, then $N_1(\bz') \in [0,2^k]$, $N_2(\bz') \geq 0$, and $N_3(\bz') \geq 2^k$. Therefore, $N'(\bz')=0=\tilde{y}$.
    \end{itemize}
\item If $\Psi(\bz')$ is an encoding of a hyperedge $S$, then:
    \begin{itemize}
    \item If $\bz'$ does not have components in the interval $(c-\frac{1}{n^2},c+\frac{2}{n^2})$, then $N_1(\bz')=P_\bx(\bz^S)$, $N_2(\bz')=N_3(\bz')=0$. Therefore, $N'(\bz')=P_\bx(\bz^S)=\tilde{y}$.
    \item If $\bz'$ has a component in the interval $(c,c+\frac{1}{n^2})$, then $N_1(\bz') \in [0,2^k]$, $N_2(\bz') \geq 0$, and $N_3(\bz') \geq 2^k$. Therefore, $N'(\bz')=0=\tilde{y}$.
    \item If $\bz'$ does not have components in the interval $(c,c+\frac{1}{n^2})$, but has a component in the interval $(c-\frac{1}{n^2},c+\frac{2}{n^2})$, then $N_1(\bz')=P_\bx(\bz^S)$ and $N_2(\bz')=0$. Therefore, $N'(\bz')=[P_\bx(\bz^S)-N_3(\bz')]_+=\tilde{y}$.
    \end{itemize}
\end{itemize}
\end{proof}

\begin{lemma}
\label{lemma:prob z' good}
Let $\tilde{\bz} \in \reals^{n^{1+3/\epsilon}}$ be the vector returned by the oracle. We have
\[
\Pr\left[\tilde{\bz} \in \tilde{\cz}\right] \geq \frac{1}{2\log(n)}~.
\]
\end{lemma}
\begin{proof}
Let $\bz' = \tilde{\bz}_{[kn]}$. We have
\begin{align*}
\Pr\left[\tilde{\bz} \in \tilde{\cz}\right]
= &\Pr\left[ \bz' \text{ does not have components in } \left(c-\frac{1}{n^2},c+\frac{2}{n^2}\right) \;\middle|\; \Psi(\bz') \text{ represents a hyperedge} \right] \cdot
\\
&\Pr\left[\Psi(\bz') \text{ represents a hyperedge} \right]~.
\end{align*}
Let $\bz = \Psi(\bz')$. By the definition of the oracle, the probability that $\bz$ is an encoding of a hyperedge, equals to the probability that a random vector whose components are drawn i.i.d. from the Bernoulli distribution encodes a hyperedge.
In the proof of Lemma~\ref{lemma:DNF distribution specific}, we showed that the probability that such vector is an encoding of a hyperedge is at least $\frac{1}{\log(n)}$. Thus, it remains to show that
\[
\Pr\left[ \bz' \text{ does not have components in } \left(c-\frac{1}{n^2},c+\frac{2}{n^2}\right) \;\middle|\; \bz \text{ represents a hyperedge} \right] \geq \frac{1}{2}~.
\]

Note that the density $\mu_-$ is bounded by $\frac{n}{2\pi}$, and that $\mu_+$ is bounded by $\frac{n}{(n-1)2\pi}$.
Hence, for a sufficiently large $n$, we have
\begin{align*}
\Pr &\left[\bz' \text{ has a component in } \left(c-\frac{1}{n^2},c+\frac{2}{n^2}\right) \;\middle|\; \bz \text{ represents a hyperedge} \right]
\\
&\leq k \cdot \frac{1}{n^2} \cdot \frac{n}{2\pi} + (nk-k) \cdot \frac{2}{n^2} \cdot \frac{n}{(n-1)2\pi}
= \frac{k}{2\pi n} + \frac{k}{\pi n}
\leq \frac{1}{2}~.
\end{align*}
\end{proof}

\subsection{Proof of Theorem~\ref{thm:DFA}}
\label{app:proof DFA}

Let $\cd$ be the uniform distribution on $\{0,1\}^{n^{1+3/\epsilon}}$. Let $c' = c \left(1 + \frac{3}{\epsilon}\right)$.
Assume that there is an efficient algorithm $\cl$ that learns DFAs with $n^3$ states on the distribution $\cd$.
Let $m(n)$ be a polynomial such that $\cl$ uses a sample of size at most $m(n)$ and returns with probability at least $\frac{3}{4}$ a hypothesis $h$ with error at most $\frac{1}{2} - \frac{1}{n^{c'}}$.
Let $s>1$ be a constant such that $n^s \geq m(n)+n^{2c'+3}$ for every sufficiently large $n$. By Assumption~\ref{ass:localPRG}, there exists a constant $k$ and a predicate $P:\{0,1\}^k \rightarrow \{0,1\}$, such that $\cf_{P,n,n^s}$ is $\frac{1}{3}$-PRG.
We will show an algorithm $\ca$ with distinguishing advantage greater than $\frac{1}{3}$ and thus reach a contradiction.

For a hyperedge $S=(i_1,\ldots,i_k)$, we denote by $\bz^S \in \{0,1\}^{nk}$ an encoding of $S$, which consists of $n$ slices of size $k$, where the $j$-th bit in the $l$-th slice is $1$ iff $l=i_j$, namely, if the index $l$ is the $j$-th member in $S$. We call $\bz^S$ the {\em short encoding} of $S$.
For $\bz \in \{0,1\}^{nk}$, we index the coordinates by $[n] \times [k]$, thus $z_{l,j}=z_{(l-1)k+j}$.
For $\tilde{\bz} \in \{0,1\}^{nk \log(n)}$, we index the coordinates by $[n] \times [k] \times [\log(n)]$, thus, $\tilde{z}_{l,j,i} = \tilde{z}_{(l-1)(k\log(n))+(j-1)\log(n)+i}$.
Let $\Psi: \{0,1\}^{nk \log(n)} \rightarrow \{0,1\}^{nk}$ be a mapping, such that $\Psi(\tilde{\bz})_{l,j}=1$ iff $\tilde{z}_{l,j,i}=1$ for every $i \in [\log(n)]$.
If $\tilde{\bz} \in \{0,1\}^{nk \log(n)}$ is such that $\Psi(\tilde{\bz}) = \bz^S$ for a hyperedge $S$, then we say that $\tilde{\bz}$ is a {\em long encoding} of $S$.
Note that a hyperedge $S$ has a single short encoding $\bz^S$, but many long encodings, since every $0$-bit in $\bz^S$ can be represented in the long encoding by any vector in the set $B = \{0,1\}^{\log(n)} \setminus \{(1,\ldots,1)\}$. Hence, given $S$, a random long encoding of $S$ can be obtained by replacing every $1$-bit in $\bz^S$ by the size-$\log(n)$ vector $(1,\ldots,1)$, and replacing every $0$-bit by a random vector from $B$.

Let $\bz \in \{0,1\}^{c'\log^2(n) \cdot nk}$ be a vector that consists of $c'\log^2(n)$ slices of size $nk$. If $\bz$ has a size-$nk$ slice that is a short encoding of a hyperedge $S$, and all preceding size-$nk$ slices do not encode hyperedges, then we say that $\bz$ is a {\em multi-short encoding} of $S$. Note that if all $c'\log^2(n)$ slices do not encode hyperedges then $\bz$ is not a multi-short encoding of any hyperedge.
For $\bz \in \{0,1\}^{c'\log^2(n)nk}$, we index the coordinates by $[c'\log^2(n)] \times [n] \times [k]$, thus $z_{d,l,j}=z_{(d-1)nk+(l-1)k+j}$.
For $\tilde{\bz} \in \{0,1\}^{c'\log^2(n)nk \cdot \log(n)}$, we index the coordinates by $[c'\log^2(n)] \times [n] \times [k] \times [\log(n)]$, thus, $\tilde{z}_{d,l,j,i} = \tilde{z}_{(d-1)nk\log(n)+(l-1)(k\log(n))+(j-1)\log(n)+i}$.
Let $\Psi': \{0,1\}^{c'\log^2(n)nk\log(n)} \rightarrow \{0,1\}^{c'\log^2(n)nk}$ be a mapping, such that $\Psi'(\tilde{\bz})_{d,l,j}=1$ iff $\tilde{z}_{d,l,j,i}=1$ for every $i \in [\log(n)]$. Thus, $\Psi'(\tilde{\bz})$ is obtained by applying $\Psi$ to every size-$nk\log(n)$ slice in $\tilde{\bz}$.
If $\tilde{\bz} \in \{0,1\}^{c'\log^2(n)nk\log(n)}$ is such that $\Psi'(\tilde{\bz})$ is a multi-short encoding of a hyperedge $S$, then we say that $\tilde{\bz}$ is a {\em multi-long encoding} of $S$.
Note that a hyperedge $S$ has a single short encoding $\bz^S$, but many multi-short encodings. Also, each multi-short encoding corresponds to many multi-long encodings.
We say that $\tilde{\bz} \in \{0,1\}^{n^{1+3/\epsilon}}$ is an {\em extended multi-long encoding} of a hyperedge $S$, if $(\tilde{z}_1,\ldots,\tilde{z}_{c'\log^3(n)nk})$ is a multi-long encoding of $S$, namely, $\tilde{\bz}$ starts with a multi-long encoding of $S$. We assume that $n^{1+3/\epsilon} \geq c'\log^3(n)nk$.
For $\bx \in \{0,1\}^n$, let $P_\bx:\{0,1\}^{n^{1+3/\epsilon}} \rightarrow \{0,1\}$ be such that for every hyperedge $S$, if $\tilde{\bz}$ is an extended multi-long encoding of $S$, then $P_\bx(\tilde{\bz}) = P(\bx_{S})$.

Let $\tilde{\bz} \in \{0,1\}^{nk\log(n)}$ be a random vector drawn from the uniform distribution. The probability that $\tilde{\bz}$ is a long encoding of a hyperedge is
\begin{align*}
n \cdot (n-1) \cdot \ldots \cdot(n-k+1) \cdot &\left(\left(\frac{1}{2}\right)^{\log(n)}\right)^{k} \left(1-\left(\frac{1}{2}\right)^{\log(n)}\right)^{nk-k}
\geq \left(\frac{n-k}{n}\right)^k \left(1-\frac{1}{n}\right)^{k(n-1)}
\\
&=\left(1-\frac{k}{n}\right)^k \left(1-\frac{1}{n}\right)^{k(n-1)}~.
\end{align*}
Since for every $x \in (0,1)$ we have $e^{-x} < 1 - \frac{x}{2}$ then for a sufficiently large $n$ the above is at least
\begin{equation*}
\exp\left(-\frac{2k^2}{n}\right) \cdot \exp\left(-\frac{2k(n-1)}{n} \right)
\geq \exp\left(-1\right) \cdot \exp\left(-2k\right)
\geq \frac{1}{\log(n)}~.
\end{equation*}
Hence, the probability that $\tilde{\bz} \sim \cd$ is an extended multi-long encoding of a hyperedge is at least
\begin{equation}
\label{eq:DFA-random encoding is hyperedge}
1 - \left(1 - \frac{1}{\log(n)} \right)^{c' \log^2(n)}
\geq 1 - \exp\left(-\frac{1}{\log(n)} \cdot c' \log^2(n) \right)
\geq 1 - \exp\left(-c' \ln(n) \right)
= 1 - \frac{1}{n^{c'}}~.
\end{equation}

Given a sequence $\cs = (S_1,y_1),\ldots,(S_{n^s},y_{n^s})$, where $S_1,\ldots,S_{n^s}$ are i.i.d. random hyperedges, the algorithm $\ca$ needs to distinguish whether $\by = (y_1,\ldots,y_{n^s})$ is random or that $\by = (P(\bx_{S_1}),\ldots,P(\bx_{S_{n^s}}))$ for a random $\bx \in \{0,1\}^n$.
We use the efficient algorithm $\cl$ in order to obtain distinguishing advantage greater than $\frac{1}{3}$ as follows.
The algorithm $\ca$ runs $\cl$ with the following examples oracle.
In the $i$-th call to the oracle, it chooses $\tilde{\bz}_i \in \{0,1\}^{n^{1+3/\epsilon}}$ according to $\cd$.
If $\tilde{\bz}_i$ is not an extended multi-long encoding of a hyperedge (with probability at most $\frac{1}{n^{c'}}$, by Eq.~\ref{eq:DFA-random encoding is hyperedge}), then the oracle returns $(\bz'_i,y'_i)$ where $\bz'_i=\tilde{\bz}_i$ and $y'_i=0$. Otherwise, the oracle chooses a random long encoding $\tilde{\bz}^{S_i}$ of $S_i$, obtains $\bz'_i$ by replacing the first size-$nk\log(n)$ slice in $\tilde{\bz}_i$ that encodes a hyperedge with $\tilde{\bz}^{S_i}$, and returns $(\bz'_i,y'_i)$ where $y'_i=y_i$. Note that the vector $\bz'_i$ returned by the oracle has the distribution $\cd$, since replacing a random long encoding of a random hyperedge with a random long encoding of another random hyperedge does not change the distribution (see Lemma~\ref{lemma:DFA oracle distribution} for a more formal proof). Let $h$ be the hypothesis returned by $\cl$.
Recall that $\cl$ uses at most $m(n)$ examples, and hence $\cs$ contains at least $n^{2c'+3}$ examples that $\cl$ cannot view. We denote the indices of these examples by $I = \{m(n)+1,\ldots,m(n)+n^{2c'+3}\}$, and denote $\cs_I = \{(S_i,y_i)\}_{i \in I}$. By $n^{2c'+3}$ additional calls to the oracle, the algorithm $\ca$ obtains the examples $\cs'_I = \{(\bz'_i,y'_i)\}_{i \in I}$ that correspond to $\cs_I$.
Let $\ell_{I}(h)=\frac{1}{|I|}\sum_{i \in I}\onefunc(h(\bz'_i) \neq y'_i)$.
Now, if $\ell_I(h) \leq \frac{1}{2}-\frac{3}{4n^{c'}}$, then $\ca$ returns $1$, and otherwise it returns $0$.
Clearly, the algorithm $\ca$ runs in polynomial time.
We now show that if $\cs$ is pseudorandom then $\ca$ returns $1$ with probability greater than $\frac{2}{3}$, and if $\cs$ is random then $\ca$ returns $1$ with probability less than $\frac{1}{3}$.

If $\cs$ is pseudorandom, then by Lemma~\ref{lemma:from P to DFA}, the examples $(\bz'_i,y'_i)$ returned by the oracle satisfy $y'_i=A(\bz'_i)$, where $A$ is a DFA with at most $n^3$ states.
Indeed, if $\bz'_i$ is an extended multi-long encoding of a hyperedge $S_i$ then $y'_i=P(\bx_{S_i})=P_\bx(\bz'_i)=A(\bz'_i)$, and otherwise $y'_i=A(\bz'_i)=0$.
Hence, if $\cs$ is pseudorandom then with probability at least $\frac{3}{4}$ the algorithm $\cl$ returns a hypothesis $h$ such that $\E_{\tilde{\bz} \sim \cd}\onefunc(h(\tilde{\bz}) \neq A(\tilde{\bz})) \leq \frac{1}{2} - \frac{1}{n^{c'}}$.
Therefore, $\E_{\cs'_I}\ell_I(h) \leq \frac{1}{2} - \frac{1}{n^{c'}}$.

If $\cs$ is random, then for the indices $i$ such that $\bz'_i$ is an extended multi-long encoding of a hyperedge, the labels $y'_i$ are independent uniform Bernoulli random variables.
Hence, for every $h$ and $i \in I$ we have
\begin{align*}
\Pr\left[h(\bz'_i) \neq y'_i\right]
&\geq \Pr\left[h(\bz'_i) \neq y'_i \;|\; \bz'_i \text{ represents a hyperedge}\right] \cdot \Pr\left[\bz'_i \text{ represents a hyperedge}\right]
\\
&\stackrel{(Eq.~\ref{eq:DFA-random encoding is hyperedge})}{\geq} \frac{1}{2} \cdot \left(1 - \frac{1}{n^{c'}}\right)
= \frac{1}{2} - \frac{1}{2n^{c'}}~.
\end{align*}
Thus, $\E_{\cs'_I}\ell_I(h) \geq \frac{1}{2} - \frac{1}{2n^{c'}}$.

By Lemma~\ref{lemma:hoefding1}, for a sufficiently large $n$, we have
\[
\Pr_{\cs'_I}\left[\left|\ell_{I}(h) -  \E_{\cs'_I}\ell_{I}(h)\right| \geq \frac{1}{4n^{c'}}\right]
\leq \Pr_{\cs'_I}\left[\left|\ell_{I}(h) -  \E_{\cs'_I}\ell_{I}(h)\right| \geq \frac{1}{n^{c'+1}}\right]
< \frac{1}{20}~.
\]
Therefore, if $\cs$ is pseudorandom, then for a sufficiently large $n$, we have with probability at least $1-\left(\frac{1}{4} + \frac{1}{20}\right) = \frac{7}{10} > \frac{2}{3}$ that $\E_{\cs'_I}\ell_{I}(h) \leq \frac{1}{2} - \frac{1}{n^{c'}}$ and $\left|\ell_{I}(h) -  \E_{\cs'_I}\ell_{I}(h)\right| < \frac{1}{4n^{c'}}$, and hence $\ell_{I}(h) \leq \frac{1}{2} - \frac{3}{4n^{c'}}$. Thus, the algorithm $\ca$ returns $1$ with probability greater than $\frac{2}{3}$.
If $\cs$ is random then $\E_{\cs'}\ell_{I}(h) \geq \frac{1}{2} - \frac{1}{2n^{c'}}$ and for a sufficiently large $n$ we have with probability at least $\frac{19}{20}$ that $\left|\ell_{I}(h) - \E_{\cs'_I}\ell_{I}(h)\right| < \frac{1}{4n^{c'}}$. Hence, with probability greater than $\frac{2}{3}$ we have $\ell_{I}(h) > \frac{1}{2} - \frac{3}{4n^{c'}}$ and the algorithm $\ca$ returns $0$.

Hence, it is hard to learn DFAs with $n^3$ states and error at most $\frac{1}{2} - \frac{1}{n^{c'}}$, where the input distribution is $\cd$.
Thus, for $\tn=n^{1+3/\epsilon}$, we have that it is hard to learn DFAs with $\tn^\epsilon = n^{(1+3/\epsilon) \cdot \epsilon} = n^{\epsilon + 3} \geq n^3$ states and error at most $\frac{1}{2} - \frac{1}{\tn^c} = \frac{1}{2} - \frac{1}{n^{(1+3/\epsilon) \cdot c}} = \frac{1}{2} - \frac{1}{n^{c'}}$, on the uniform distribution over $\{0,1\}^\tn$.

\begin{lemma}
\label{lemma:DFA oracle distribution}
The distribution of an example $\bz' \in \{0,1\}^{n^{1+3/\epsilon}}$ returned by the oracle is $\cd$.
\end{lemma}
\begin{proof}
Let $\tn = n^{1+3/\epsilon}$.
Recall that the oracle first chooses $\tilde{\bz} \in \{0,1\}^{\tn}$ according to $\cd$. If $\tilde{\bz}$ is not an extended multi-long encoding of a hyperedge then $\bz'=\tilde{\bz}$. Otherwise, the oracle chooses a random long encoding $\tilde{\bz}^{S}$ of a random hyperedge $S$, and obtains $\bz'$ by replacing the first size-$nk\log(n)$ slice in $\tilde{\bz}$ that encodes a hyperedge with $\tilde{\bz}^{S}$.

Let $\bz^0 \in \{0,1\}^{\tn}$. We show that the probability of $\bz'=\bz^0$ is $\frac{1}{2^\tn}$.
If $\bz^0$ is not an extended multi-long encoding of a hyperedge, then $\Pr[\bz'=\bz^0]=\Pr[\tilde{\bz}=\bz^0]=\frac{1}{2^\tn}$.
Assume that $\bz^0$ is an extended multi-long encoding of a hyperedge $S_0$, and the first size-$nk\log(n)$ slice in $\bz^0$ that encodes a hyperedge is the $d$-th slice, for some $d \in [c'\log^2(n)]$. Thus, $z^0_{(d-1)nk\log(n)+1},\ldots,z^0_{dnk\log(n)}$ is a long encoding of $S_0$.
Let $B_0 \subseteq \{0,1\}^{\tn}$ be the set of all extended multi-long encodings
that can be obtained from $\bz^0$ by replacing the $d$-th size-$nk\log(n)$ slice with some long encoding of some hyperedge.
Note that $\bz'=\bz^0$ iff the oracle chooses $\tilde{\bz} \in B_0$, and then chooses $S=S_0$, and then chooses the long encoding $\tilde{\bz}^S=z^0_{(d-1)nk\log(n)+1},\ldots,z^0_{dnk\log(n)}$.
For every $\bz \in B_0$, by replacing the $d$-th size-$nk\log(n)$ slice with a random long encoding of a random hyperedge, we obtain a random (uniformly distributed) vector in $B_0$.
Hence, $\bz'=\bz^0$ iff we have: (1) the oracle first chooses $\tilde{\bz} \in B_0$, (2) the oracle chooses $\bz^0$ as the random vector in $B_0$.
Therefore, we have
\[
\Pr\left[\bz'=\bz^0\right] = \frac{|B_0|}{2^\tn} \cdot \frac{1}{|B_0|} = \frac{1}{2^\tn}~.
\]
\end{proof}

\begin{lemma}
\label{lemma:DFA chekcs encoding}
For a sufficiently large $n$, there exists a DFA $A_E$ with at most $\log(n)$ states such that $A_E$ accepts a word $\bz \in \{0,1\}^{nk}$ iff $\bz$ is a short encoding of a hyperedge.
\end{lemma}
\begin{proof}
A word $\bz \in \{0,1\}^{nk}$ is a short encoding of a hyperedge iff the following conditions hold:
\begin{itemize}
\item Every size-$k$ slice in $\bz$ includes at most one $1$-bit.
\item There are no two size-$k$ slices in $\bz$ that have $1$-bit in the same index (and thus correspond to the same member in the hyperedge).
\item For every $j \in [k]$ there is a size-$k$ slice in $\bz$ with $1$-bit in index $j$.
\end{itemize}

We construct a DFA  $A_E = \tup{\Sigma,Q,q_0,\delta,F}$ that checks these conditions. We have $\Sigma=\{0,1\}$, $Q = \{q_{\text{rej}}\} \cup ([k] \times \{0,1\} \times 2^{[k]})$, $q_0 = (1,0,\emptyset)$,
and $F = \{(1,0,[k])\}$. Note that $Q$ is of size at most $\log(n)$ (for a sufficiently large $n$).
The states in $Q$ are such that the first component keeps the current location in the size-$k$ slice, the second component keeps whether a $1$-bit already appeared in the current slice, and the third component keeps the subset of indices in $[k]$ that are already occupied.
For $i \in [k-1]$, $b \in \{0,1\}$ and $I \subseteq [k]$, we have
\begin{itemize}
    \item $\delta((i,b,I),0) = (i+1,b,I)$.
    \item $\delta((k,b,I),0) = (1,0,I)$.
    \item $\delta((i,1,I),1) = \delta((k,1,I),1) = q_{\text{rej}}$.
    \item $\delta((i,0,I),1) = (i+1,1,I \cup \{i\})$ if $i \not \in I$, and $\delta((i,0,I),1) = q_{\text{rej}}$ otherwise.
    \item $\delta((k,0,I),1) = (1,0,I \cup \{k\})$ if $k \not \in I$, and $\delta((k,0,I),1) = q_{\text{rej}}$ otherwise.
    \item $\delta(q_{\text{rej}},0) = \delta(q_{\text{rej}},1) = q_{\text{rej}}$.
\end{itemize}
\end{proof}

\begin{lemma}
\label{lemma:DFA chekcs P}
For every $\bx \in \{0,1\}^n$ and a sufficiently large $n$, there is a DFA $A_P$ with at most $n\log(n)$ states such that $A_P$ accepts a short encoding of a hyperedge $S$ iff $P(\bx_S)=1$.
\end{lemma}
\begin{proof}
Let $\bz^S$ be a short encoding of a hyperedge $S$.
We construct $A_P = \tup{\Sigma,Q,q_0,\delta,F}$ that accepts $\bz^S$ iff $P(\bx_S)=1$.
Let $\Sigma=\{0,1\}$, let $B = \{0,1,\_\}^k$ and $Q=\{q_0\} \cup ([n] \times [k] \times B)$, and let $F = \{n\} \times \{k\} \times \{\bb \in \{0,1\}^k: P(\bb)=1\}$.
Note that $Q$ is of size at most $n\log(n)$ (for a sufficiently large $n$).
The states in $Q$ are such that the first two components keep the current location in the short encoding, and the third component keeps the information on $\bx_S$.
The transitions are
\begin{itemize}
    \item $\delta(q_0,0) = (1,1,(\_,\ldots,\_))$.
    \item $\delta(q_0,1) = (1,1,(x_1,\_,\ldots,\_))$.
    \item For $i \in [n]$, $j \in [k-1]$, $\bb \in B$ we have:
    \begin{itemize}
        \item $\delta((i,j,\bb),0)=(i,j+1,\bb)$.
        \item $\delta((i,j,\bb),1)=(i,j+1,(b_1,\ldots,b_j,x_i,b_{j+2},\ldots,b_k))$.
        \item $\delta((i,k,\bb),0)=((i \mod n)+1,1,\bb)$.
        \item $\delta((i,k,\bb),1)=((i \mod n)+1,1,(x_{(i \mod n)+1},b_2,\ldots,b_k))$.
    \end{itemize}
\end{itemize}
\end{proof}

\begin{lemma}
\label{lemma:from P to DFA}
For every $\bx \in \{0,1\}^n$ and a sufficiently large $n$, there is a function $f:\{0,1\}^{n^{1+3/\epsilon}} \rightarrow \{0,1\}$ that can be expressed by a DFA with at most $n^3$ states, such that:
\begin{itemize}
\item For every hyperedge $S$ and every extended multi-long encoding $\tilde{\bz}$ of $S$, we have $f(\tilde{\bz}) = P_\bx(\tilde{\bz})$.
\item For every $\tilde{\bz} \in n^{1+3/\epsilon}$ that is not an extended multi-long encoding of a hyperedge, we have $f(\tilde{\bz}) = 0$.
\end{itemize}
\end{lemma}
\begin{proof}
Let $d \geq c'\log^2(n) \cdot nk$. We first construct a DFA $A'$ such that for every $\bz \in \{0,1\}^d$ we have: If $\bz$ starts with a multi-short encoding of a hyperedge $S$ then $A'$ accepts $\bz$ iff $P(\bx_{S})=1$, and if $\bz$ does not start with a multi-short encoding of a hyperedge then $A'$ rejects $\bz$.
Let $A_E$ and $A_P$ be the DFAs from Lemmas~\ref{lemma:DFA chekcs encoding} and~\ref{lemma:DFA chekcs P}. Thus, $A_E$ checks whether a word is a short encoding of a hyperedge, and $A_P$ checks whether a short encoding $\bz^S$ is such that $P(\bx_S)=1$.
The DFA $A'$ runs $A_E$ and $A_P$ in parallel on the first size-$nk$ slice. If both $A_E$ and $A_P$ accept then $A'$ accepts, if $A_E$ accepts and $A_P$ rejects then $A'$ rejects, and if $A_E$ rejects then $A'$ continues to the next size-$nk$ slice in a similar manner. Also, $A'$ keeps a counter and stops after $c'\log^2(n)$ slices.
Constructing such a DFA is straightforward. Moreover, since $A_E$ has at most $\log(n)$ states and $A_P$ has at most $n\log(n)$ states, then $A'$ has at most $\log(n) \cdot n\log(n) \cdot (c'\log^2(n)nk+1) \leq n^2\log^5(n)$ states (for a sufficiently large $n$).

Next, we construct a DFA $A$ such that for every $\tilde{\bz} \in \{0,1\}^{n^{1+3/\epsilon}}$ we have: If $\tilde{\bz}$ starts with a multi-long encoding of a hyperedge $S$ then $A$ accepts $\tilde{\bz}$ iff $P(\bx_{S})=1$, and if $\tilde{\bz}$ does not start with a multi-long encoding of a hyperedge then $A$ rejects $\tilde{\bz}$.
Note that such a DFA $A$ satisfies the lemma's requirements.
The DFA $A$ is obtained from $A'$ by replacing each state $q$ in $A'$ by the DFA $A^q=\tup{\Sigma,Q^q,q,\delta^q,F^q}$ such that $Q^q = \{q\} \cup ([\log(n)-1] \times \{0,1\})$, $\delta^q(q,0) = (1,0)$, $\delta^q(q,1) = (1,1)$, and for every $i \in [\log(n)-2]$ we have $\delta^q((i,1),1)=(i+1,1)$, and $\delta^q((i,0),0)=\delta^q((i,0),1)=\delta^q((i,1),0)=(i+1,0)$. Then, for the transitions $\delta'(q,0)=q'$ and $\delta'(q,1)=q''$ in the DFA $A'$, the DFA $A$ includes the appropriate transitions from the states $(\log(n)-1,0)$ and $(\log(n)-1,1)$ of $A^q$, namely, $\delta((\log(n)-1,1),0)=\delta((\log(n)-1,0),0)=\delta((\log(n)-1,0),1)=q'$ and $\delta((\log(n)-1,1),1)=q''$. Also, if $q$ is an accepting state in $A'$ then we set $F^q=\{q\}$ and otherwise $F^q=\emptyset$. Thus, $A'$ and $A$ have the same accepting states. Note that $A$ has at most $n^2\log^5(n) \cdot 2\log(n) \leq n^3$ states.
\end{proof}

\subsection{Proof of Theorem~\ref{thm:sparse poly GF2}}
\label{app:proof sparse poly GF2}

In the proof of Theorem~\ref{thm:DNF}, we constructed a DNF formula $\psi_\bx$ such that for every encoding $\bz^S \in \{0,1\}^{kn}$ of a hyperedge $S$ we have $\psi_\bx(\bz^S)=P_\bx(\bz^S)=P(\bx_S)$. We now show that there is a $2^k$-sparse $GF(2)$ polynomial $h:\{0,1\}^{kn} \rightarrow \{0,1\}$, such that for every hyperedge $S$ we have $h(\bz^S)=P_\bx(\bz^S)$. Namely, $h$ agrees with $\psi_\bx$ on inputs that encode hyperedges.
Then, the theorem follows from the arguments in the proof of Theorem~\ref{thm:DNF}.

By Lemma~\ref{lemma:from P to DNF}, the DNF $\psi_\bx$ has at most $2^k$ terms. Each term $C_j$ in $\psi_\bx$ is a conjunction of positive literals, such that $C_j(\bz^S)=1$ iff $\bx_S$ is the $j$-th satisfying assignment of the predicate $P$. Hence, it is not possible that more than one term in $\psi_\bx(\bz^S)$ is satisfied. Let $h$ be the $GF(2)$ polynomial induced by $\psi_\bx$, i.e., each monomial in $h$ corresponds to a term $C_j$ from $\psi_\bx$. Since at most one term in $\psi_\bx(\bz^S)$ is satisfied, then we have: 
If $\psi_\bx(\bz^S)=1$ then exactly one term in $\psi_\bx(\bz^S)$ is satisfied, and therefore $h(\bz^S)=1$. Also, if $\psi_\bx(\bz^S)=0$ then all terms in $\psi_\bx(\bz^S)$ are unsatisfied, and therefore $h(\bz^S)=0$.

\end{document}